\documentclass{article} 
\usepackage{iclr2027_conference,times}
\iclrfinalcopy

\usepackage{amsmath,amsfonts,bm}

\def\eqref#1{equation~\ref{#1}}

\def\1{\bm{1}}

\DeclareMathAlphabet{\mathsfit}{\encodingdefault}{\sfdefault}{m}{sl}
\SetMathAlphabet{\mathsfit}{bold}{\encodingdefault}{\sfdefault}{bx}{n}

\def\gE{{\mathcal{E}}}

\DeclareMathOperator*{\argmax}{arg\,max}
\DeclareMathOperator*{\argmin}{arg\,min}

\usepackage{hyperref}
\usepackage{url}
\usepackage{amsthm}
\usepackage{graphicx}
\usepackage{caption}
\usepackage{subcaption}
\usepackage{algorithmic}
\usepackage{algorithm}
\usepackage{enumitem}
\usepackage{booktabs}

\usepackage{amsmath}
\usepackage{bbm}
\usepackage{multirow}

\title{Rising Multi-Armed Bandits with Known \\Horizons}

\author{
\parbox{\dimexpr\textwidth-2\tabcolsep\relax}{\centering
Seockbean Song$^{1}$,\! Chenyu Gan$^{2}$,\! Youngsik Yoon$^{3}$,\! 
Siwei Wang$^{4}$,\! Wei Chen$^{4}$,\! Jungseul Ok$^{1,3}$\thanks{Correspondence to: jungseul.ok@postech.ac.kr} \\ [0.3em]
\mdseries
$^{1}$ Graduate School of AI, POSTECH,
$^{2}$ Qiuzhen College, Tsinghua University, \\ [0.3em]
$^{3}$ Department of CSE, POSTECH, 
$^{4}$ Microsoft Research Asia \\ [0.3em]
\text{\{shinebobo, ysyoon97, jungseul.ok\}@postech.ac.kr}, \\[0.3em] 
\text{gancy22@mails.tsinghua.edu.cn}, 
\text{\{siweiwang, weic\},microsoft.com}
}}

\theoremstyle{plain}
\newtheorem{theorem}{Theorem}[section]
\newtheorem{proposition}[theorem]{Proposition}
\newtheorem{lemma}[theorem]{Lemma}

\theoremstyle{definition}
\newtheorem{definition}[theorem]{Definition}
\newtheorem{assumption}[theorem]{Assumption}
\theoremstyle{remark}

\newtheorem{claim}{Claim}
\renewcommand\vec{\boldsymbol}
\newcommand{\EXP}{\mathbb{E}} 
\renewcommand{\Pr}{\mathbb{P}} 
\newcommand{\ind}{\mathbbm{1}}

\definecolor{purp}{rgb}{0.65, 0.16, 0.65}

\begin{document}

\maketitle
\lhead{}
\renewcommand{\headrulewidth}{0pt}

\begin{abstract} \label{sec:abstract}
Rising Multi-Armed Bandits (RMABs) model sequential decision problems where each arm's expected reward improves with repeated pulls. 
In such problems, the value of investing in an arm depends on how much time remains, making knowledge of the horizon useful side information, yet its benefit remains underexplored.
We investigate this benefit through CURE-UCB, a horizon-aware algorithm that estimates each arm's cumulative reward over the remaining horizon.
Theoretically, under structured assumptions, we prove that CURE-UCB uniformly dominates a representative horizon-agnostic algorithm and show that the advantage of horizon awareness can be substantial: on some instances, CURE-UCB incurs only $O(1)$ regret whereas the horizon-agnostic algorithm suffers $\Omega(T)$.
Furthermore, we establish a regret upper bound for the general concave rising bandit setting whose growth-dependent term matches the known lower bound in its dependence on $T$.
Empirically, across synthetic benchmarks and real-world model selection tasks, CURE-UCB achieves lower regret than both rising and non-stationary baselines over a wide range of horizons.


\end{abstract}

\section{Introduction} \label{sec:intro}


Rising Multi-Armed Bandits (RMABs) model environments where an arm's expected reward improves with repeated trials, capturing learning phenomena in applications such as hyperparameter optimization \citep{jamieson2016non,li2018hyperband, li2020efficient, falkner2018bohb}, robot skill acquisition \citep{song2026combinatorial, yoon2024beag, mehta2024waypoint, portelas2020teacher}, and online crowdsourcing tasks \citep{sawwan2024diversity, song2021minimizing}. In such practical scenarios, a learner optimizes the performance under a fixed and known horizon $T$, such as a number of training epochs. 
This known horizon is the extra information the learner can bring to each decision. 
\begin{figure}[!t]
    \vspace{-0.5cm}
    \centering
    \captionsetup[subfigure]{skip=0pt, belowskip=3pt}

    \begin{subfigure}[t]{0.48\linewidth}
        \centering
        \includegraphics[
            width=0.85\linewidth,
            trim={0cm 0.5cm 0cm 0.0cm},
            clip
        ]{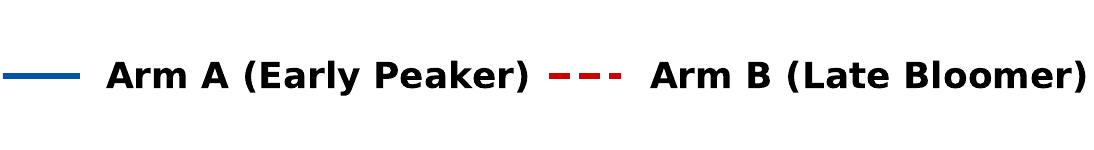}
        
        \includegraphics[width=\linewidth]{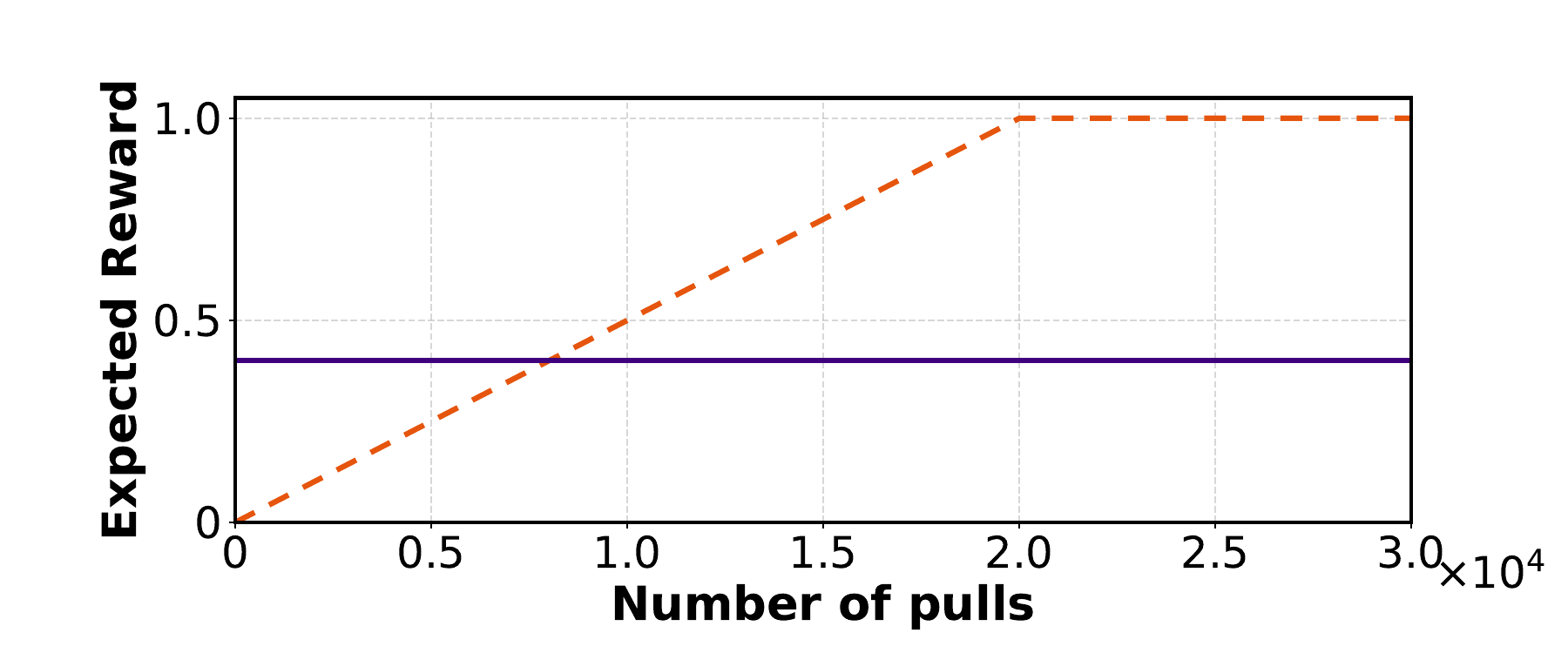}
        
        \caption{Expected reward functions}
        \label{fig:intro_reward_wide}
    \end{subfigure}
    \hfill
    \begin{subfigure}[t]{0.48\linewidth}
        \centering
        \includegraphics[
            width=0.85\linewidth,
            trim={0cm 0.0cm 0cm 0.0cm},
            clip
        ]{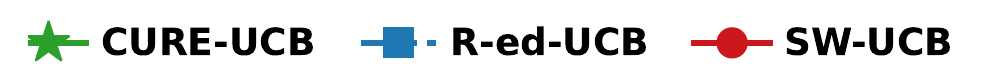}
                
        \includegraphics[width=\linewidth]{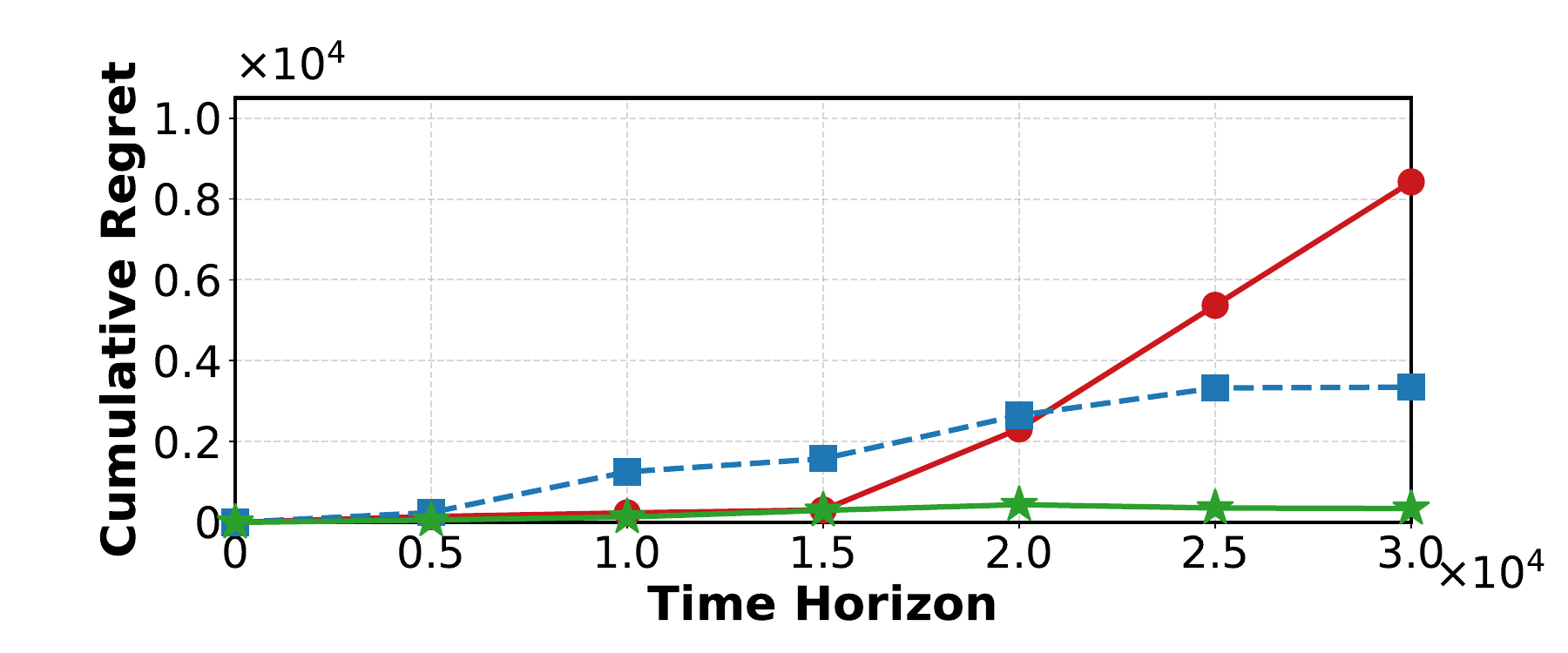}
        
        \caption{Cumulative regret for different time horizons $T$}
        \label{fig:intro_regret_T}
    \end{subfigure}

    \begin{subfigure}[t]{0.48\linewidth}
        \centering
        
        \includegraphics[width=\linewidth]{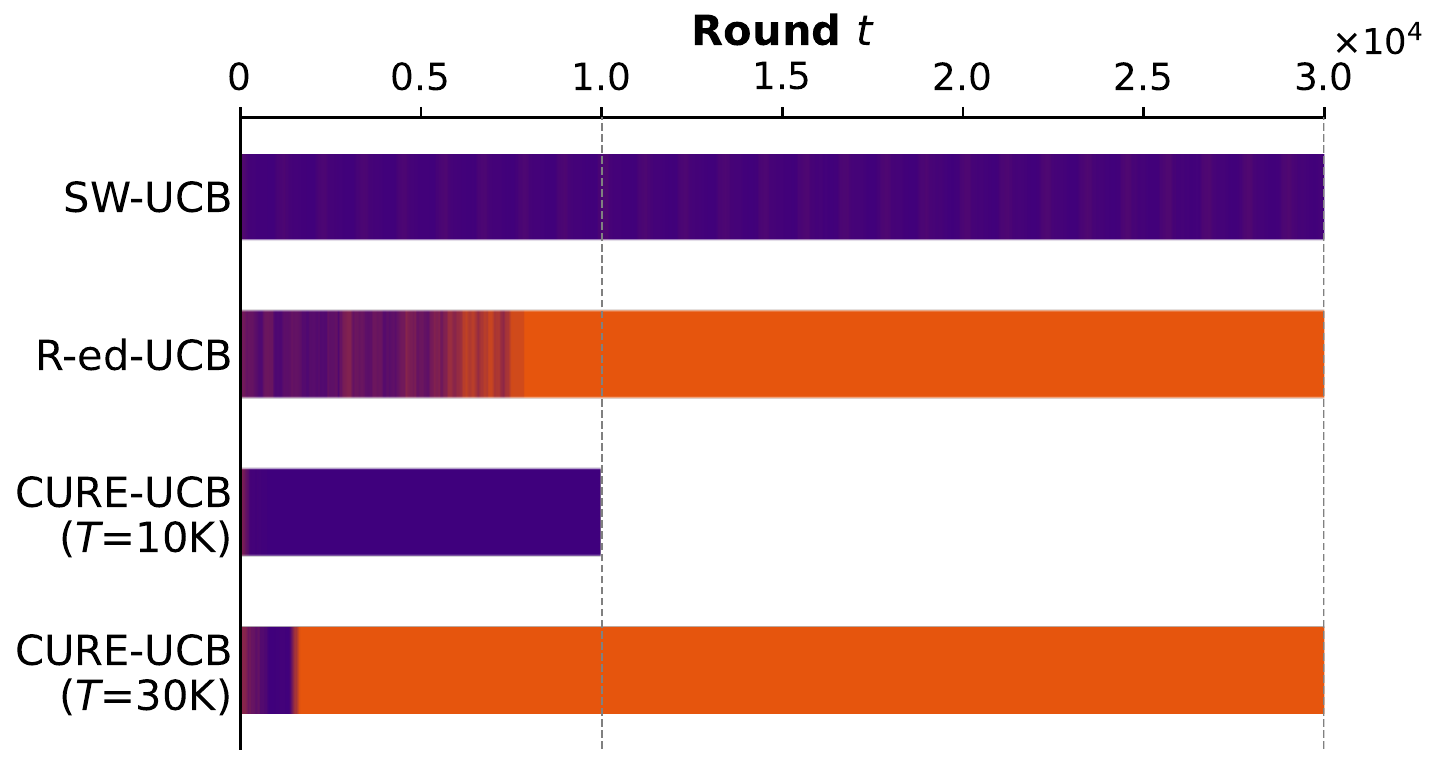}

        \caption{Selection trajectories}
        \label{fig:intro_trajectory}
    \end{subfigure}
    \hfill
    \begin{subfigure}[t]{0.48\linewidth}
        \centering
        \makebox[\linewidth][c]{%
            \hspace*{0.2\linewidth}
            \includegraphics[width=0.98\linewidth]{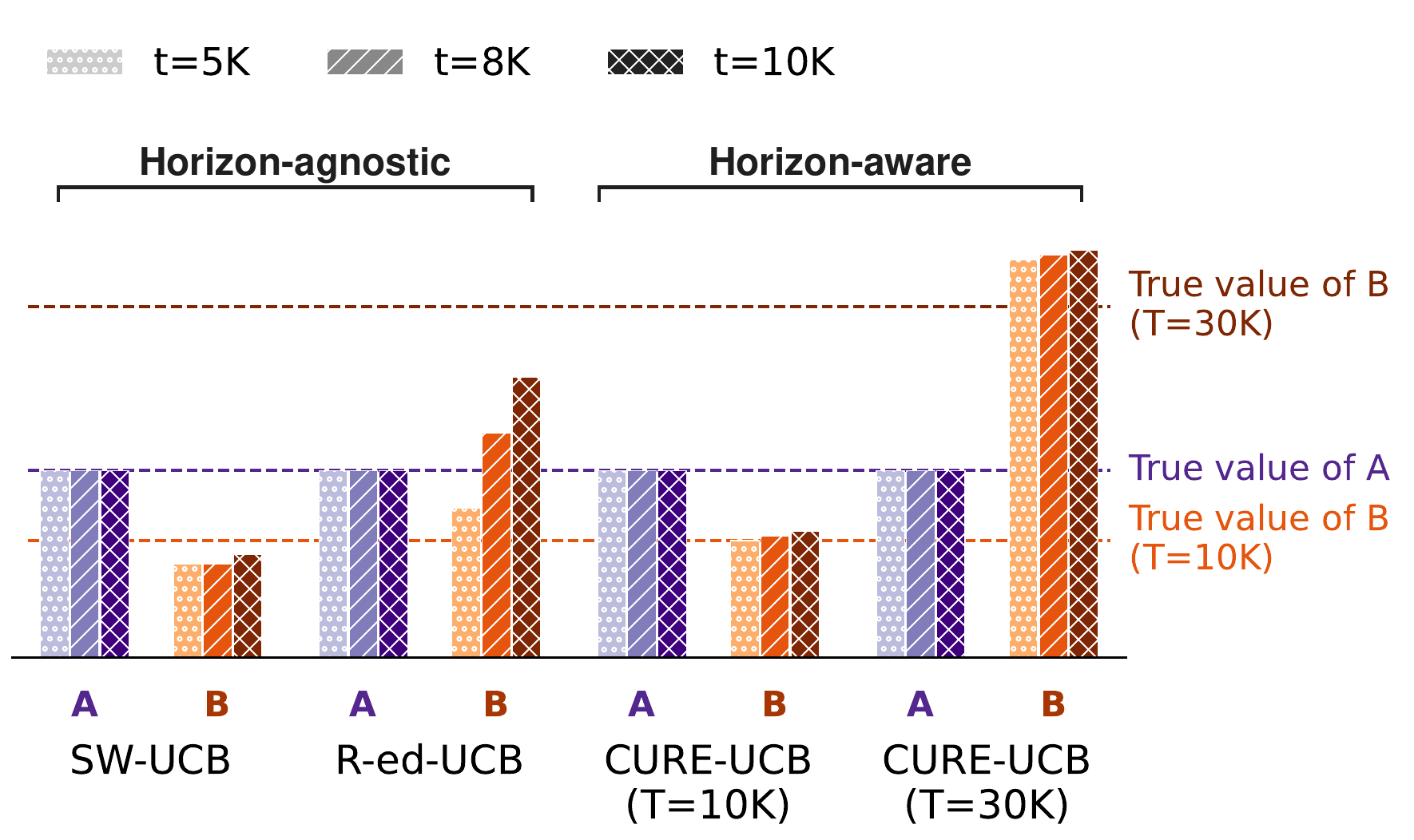}%
        }
        
        \caption{Decision criteria at different time steps $t$}
        \label{fig:intro_criteria}
    \end{subfigure}

    \vspace{-0.2cm}

    \caption{\textbf{Demonstration of Horizon-Adaptiveness.} 
    (a) Expected reward functions.
    \textcolor[HTML]{3F007D}{\textbf{Arm A}} represents an \textit{Early Peaker} 
    (high initial reward, limited growth), and 
    \textcolor[HTML]{E6550D}{\textbf{Arm B}} represents a \textit{Late Bloomer} 
    (low initial reward, high potential).
    (b) Cumulative regret for different time horizons $T$, showing different regret trends among \textcolor{green!50!black}{\textbf{CURE-UCB}}, \textcolor[HTML]{1F77B4}{\textbf{R-ed-UCB}}, and \textcolor{red!80!black}{\textbf{SW-UCB}}.
    (c) Selection trajectories: Horizon-agnostic baselines (R-ed-UCB and SW-UCB) follow a fixed selection trajectory regardless of the horizon $T$. In contrast, CURE-UCB adapts its trajectory to the horizon.
    (d) Decision criteria at different time steps $t$: SW-UCB focuses on current rewards, whereas R-ed-UCB focuses on a specific future point.
    CURE-UCB instead uses a horizon-aware criterion, evaluating the same arms differently depending on remaining time. 
    }
    \vspace{-0.3cm}
    \label{fig:intro}
\end{figure}

Figure~\ref{fig:intro} illustrates why knowledge of the horizon can be beneficial. 
Even with the same reward functions (Figure~\ref{fig:intro_reward_wide}), the value of investing in an arm depends on the horizon over which its improving rewards can be realized; i.e., an early peaker is preferable for a short horizon, whereas a late bloomer is preferable for a long horizon. 
This horizon dependence results in distinct regret patterns across horizons: as shown in Figure~\ref{fig:intro_regret_T}, R-ed-UCB incurs large regret at short horizons ($T\!=\!10K$), whereas SW-UCB performs poorly at longer horizons ($T\!=\!30K$), while CURE-UCB maintains low regret across the considered horizons.

Figures~\ref{fig:intro_trajectory} and ~\ref{fig:intro_criteria} explain why these regret patterns arise.
Figure~\ref{fig:intro_trajectory} shows the different selection trajectories. SW-UCB continues to select the early peaker, whereas R-ed-UCB switches from the early peaker to the late bloomer and then commits to it. These trajectories remain unchanged across the horizons as the decision rules of both algorithms are independent of $T$. 
In contrast, CURE-UCB follows different trajectories for $T\!=\!10K$ and $T\!=\!30K$, adapting its selections to the given horizon. 
Figure~\ref{fig:intro_criteria} illustrates the difference in their decision criteria. 
SW-UCB focuses on current rewards and therefore tends to select the early peaker, whereas R-ed-UCB extrapolates each arm's reward to a specific future point and switches to the late bloomer once its projected reward becomes larger.
Although these criteria capture different aspects of an arm's reward trajectory, each can lead to suboptimal selections depending on $T$. 

Knowing $T$ in advance could help to avoid such horizon-dependent suboptimal selections.
CURE-UCB, our proposed CUmulative Reward Estimation UCB, is designed to put this information to use.
At each round it uses the remaining horizon to estimate how much cumulative reward each arm could deliver if committed to for the rest of the horizon.
As illustrated in Figure~\ref{fig:intro_criteria}, this allows CURE-UCB to evaluate the same arms differently depending on $T$, leading to horizon-dependent decisions.
We note that our algorithm aligns with the optimality criterion of RMAB \citep{heidari2016tight}, presented in Proposition~\ref{prop:optimal_structure}.

To characterize the advantage of horizon awareness, we first study Linear-Then-Flat (LTF) reward instances, a structured setting that captures the increasing-then-converging shape commonly observed in practical tasks such as model selection and hyperparameter tuning \citep{xia2024llm, li2018hyperband}.
In this setting, we show that CURE-UCB uniformly dominates R-ed-UCB, a representative horizon-agnostic rising bandit algorithm \citep{metelli2022stochastic}  for every finite horizon (Theorem~\ref{thm:dominance_ltf}).
This result provides a controlled comparison showing that exploiting horizon information can consistently improve performance in this structured setting.
Moreover, the advantage can be substantial rather than marginal: on some instances, CURE-UCB incurs only $O(1)$ regret while R-ed-UCB suffers $\Omega(T)$ (Proposition~\ref{prop:linear_gap}).
Beyond the deterministic LTF setting, we establish a regret bound for CURE-UCB under general concave rising rewards  (Theorem~\ref{thm:upper_bound}) and show that its growth-dependent term matches the known lower bound \citep{fiandri2024rising}. 

Empirically, we evaluate CURE-UCB across a broad range of synthetic and real-world environments. 
Our synthetic experiments cover both LTF dynamics and diverse concave reward functions, allowing us to examine its behavior under different reward-growth patterns and horizon lengths. 
Furthermore, we consider real-world model-selection tasks, including online model selection on IMDB \citep{maas2011learning} and LLM selection \cite{xia2024llm}, where model performance improves with additional training.
Across these settings, CURE-UCB consistently achieves lower regret than both rising and non-stationary baselines over a wide range of horizons.


Our main contributions are summarized as follows:
\vspace{-0.2cm}
\begin{itemize}[left=0.1cm]
    \item  We propose CURE-UCB, a horizon-aware algorithm that directly incorporates the remaining horizon into arm evaluation. (Section~\ref{sec:algorithm}).
    \item We theoretically characterize the benefit of horizon awareness, establishing uniform dominance in a structured setting, an order-level regret gap on some instances, and a general regret bound (Section~\ref{sec:analysis}).
    \item We demonstrate the practical benefit of CURE-UCB across diverse synthetic and real-world model selection tasks, consistently outperforming rising and non-stationary baselines over a wide range of horizons. (Section~\ref{sec:experiments}).
\end{itemize}
\section{Related work} \label{sec:rel}

The RMAB setting is a special case of rested non-stationary bandits \citep{tekin2012online, levine2017rotting}, where the reward of an arm shifts only with the number of times it is selected. Standard non-stationary methods based on forgetting \citep{garivier2011upper, trovo2020sliding} or restarts \citep{allesiardo2017non, besbes2014stochastic} are designed to track a drifting current mean, and do not exploit the monotone rested structure of RMAB.

Rising bandit algorithms \citep{metelli2022stochastic, fiandri2025thompson, xia2024llm, patil2022mitigating, fiandri2024rising, blum2025nearlytight} instead capture this rising trend.
Most are horizon-agnostic, targeting a surrogate such as the instantaneous reward at a specific future time, which is a sound objective as the horizon grows large but is no longer aligned with the cumulative reward attainable in the time left once $T$ is finite.

A separate line of work explicitly incorporates the horizon $T$, and our contribution is situated here. 
\citet{heidari2016tight} characterize the horizon-dependent optimal policy, but only consider deterministic rewards and provide a sublinear policy-regret guarantee, but not the stochastic finite-sample regret guarantee considered here.
\citet{amichay2025rising} provide a tight regret bound with R-ed-AE, but confine it only to linear drift.
In both studies, horizon awareness is a means to an end rather than an object of study, and the quantitative benefit of knowing $T$ is left open. 
CURE-UCB is built to answer exactly this: through both theoretical and empirical analysis, it isolates the effect of horizon awareness that prior work leaves implicit.

\section{Problem Formulation} \label{sec:pf}
\label{sec:problem_formulation}

We consider the RMAB problem, where the expected reward of each arm increases with the number of times it has been selected. 
Let $K$ be the number of arms, indexed by $[K]:=\{1,\dots,K\}$. 
At each round $t \in [T]$, the learner selects an arm $I_t \in [K]$ and observes a reward $X_t$. This reward is drawn independently from a distribution $D_{I_t}(N_{I_t,t-1}+1)$, where $N_{i,t-1}$ denotes the number of times arm $i$ has been selected up to time $t-1$. 

We assume that the reward distribution $D_i(n)$ is $\sigma^2$-subgaussian with a known variance parameter $\sigma^2$. 
Let $\mu_i(n):=\mathbb{E}_{X\sim D_i(n)}[X]$ denote the expected reward of arm $i$ on its $n$-th pull, assuming $\mu_i(n) \in [0,1]$ for all $i\in [K]$ and $n\in [T]$. 
We denote the problem instance by the vector of mean reward functions $\vec{\mu} := \left(\mu_1,\mu_2,\dots,\mu_K\right)$.
A defining characteristic of the rising bandit setting is that, for all arms $i \in[K]$, the mean rewards are non-decreasing in the number of plays $n$:
\begin{align}
    \gamma_i(n) := \mu_i(n+1)-\mu_i(n)\ge 0,  \quad n \geq 1 \;.
\end{align}

The learner's history up to time $t$ is denoted by $\mathcal{H}_{t} := \{(I_{\tau}, X_{\tau}): \tau \in [t]\}$. 
A policy $\pi$ maps this history $\mathcal{H}_{t-1}$ to an arm $I_t \in [K]$ (or a distribution on arms) to be selected at round $t$.

For analytical tractability, consistent with the rising bandit literature \citep{heidari2016tight, metelli2022stochastic, patil2022mitigating}, we assume the concavity of the mean reward function $\mu_i$. 
This implies that while the expected rewards increase, the marginal gain from each additional pull is non-increasing.

\begin{assumption} \label{asu:concavity} (Concavity)
For all arms $i\in [K]$ and number of plays $n \ge 1$, the increments are non-increasing:
\begin{align}
    \gamma_i(n) \ge \gamma_i(n+1).
\end{align}
\end{assumption}
\paragraph{Regret Minimization}
The objective of the learner is to maximize the cumulative reward over a finite horizon $T$. 
Let $V(\pi, T) := \mathbb{E}_{\pi}\bigl[\sum_{t=1}^{T} X_t\bigr]$ be the expected cumulative reward of a policy $\pi$. 
We define the optimal policy $\pi^*$ as the policy that maximizes the expected cumulative reward:
\begin{align}
    \pi^* := \argmax_{\pi} V(\pi, T).
\end{align}
The performance of a policy $\pi$ is measured by its regret, defined as the difference between the expected cumulative reward of the optimal policy and that of $\pi$: 
\begin{align}
    Reg_{\vec{\mu}}(\pi,T) := V(\pi^*, T) - V(\pi, T).
\end{align}
It is crucial to distinguish the regret definition employed here from that used in standard adversarial bandits. Standard adversarial bandit algorithms typically optimize for \textit{external regret} \citep{auer2002nonstochastic, blum2007external}, which measures the learner against the best policy on the trajectory it actually ran.
However, this metric is conceptually inadequate when rewards
depend on past actions, as in the RMAB setting, because it ignores
higher rewards attainable on other trajectories.
For this reason, \textit{policy regret} \citep{arora2012online} has been established as the standard metric in rising bandits for regret minimization \citep{metelli2022stochastic, song2026combinatorial, genalti2024graph, heidari2016tight}.
\textit{Policy regret} instead compares the learner against the optimal policy $\pi^*$ 
over all possible trajectories, i.e., the highest expected cumulative reward attainable over the horizon $T$ by any policy, independent of the trajectory the learner actually ran.
While identifying $\pi^*$ generally involves dynamic planning, the non-decreasing reward structure simplifies the optimal policy. Specifically, the optimal policy for a given $T$ reduces to selecting a single arm that maximizes cumulative reward throughout the horizon \citep{heidari2016tight}:

\begin{proposition}[Horizon-dependent optimality, \citet{heidari2016tight}] \label{prop:optimal_structure}
For an RMAB problem with a finite horizon $T$, the optimal policy $\pi^*$ consists of selecting the same arm $i^* \in [K]$ at all $t\in [T]$, where $i^*$ is the arm maximizing the cumulative reward over the horizon $T$:
\begin{align}
i^* = \argmax_{i \in [K]} \sum_{t=1}^{T} \mu_i(t) \;.
\end{align}
\end{proposition}
This proposition shows that the optimal policy depends on $T$. 
More precisely, it selects a single arm $i^*$ throughout the horizon, and this arm depends on $T$, so $i^*$ can differ across horizons. 

\section{Algorithm}\label{sec:algorithm}

Motivated by the theoretical insight that the optimal policy commits to the single arm with the largest cumulative reward over the horizon, we propose the CUmulative Reward Estimation UCB (CURE-UCB) algorithm. 
By leveraging the additional information of the horizon $T$, CURE-UCB aligns its selection with the criterion of Proposition~\ref{prop:optimal_structure}.
The complete procedure is outlined in Algorithm~\ref{alg:cure}.

\setlength{\intextsep}{10pt}
\begin{algorithm}[h]
\caption{CURE-UCB $\pi^{\mathrm{cure}}_\epsilon$}
\label{alg:cure}
\begin{algorithmic}[1]
    \STATE Input: Horizon $T$, Number of arms $K$, sliding window parameter  $\epsilon$
    \STATE Initialize: Select each arm $i \in [K]$ for $\left\lceil\frac{2}{\epsilon} \right\rceil$ times.
    \FOR{$t = \left\lceil\frac{2}{\epsilon} \right\rceil K + 1, \dots, T$}
        \STATE Compute index $B_i(t)$ for each arm, where $B_i(t)$ is defined in \eqref{eq:cure_index}\\
        \STATE Select arm $I_t = \argmax_{i \in [K]} B_i(t)$ and observe $X_t$.
        \STATE Update estimates and counter $N_{I_t} \leftarrow N_{I_t} + 1$.
    \ENDFOR
\end{algorithmic}
\end{algorithm}

\subsection{The Horizon-Adaptive Index}
We now describe the construction of the horizon-adaptive index $B_i(t)$ in detail.
Let $X_i(l)$ denote the reward observed when arm $i$ is selected for the $l$-th time.
The index is formally defined as the sum of three distinct components:
\begin{align}
    \label{eq:cure_index}
    B_i(t) := &\underbrace{\frac{1}{h_i}\sum_{l=N_{i,t-1}-h_i+1}^{N_{i,t-1}}X_i(l)}_{\text{(i) Recent average}} 
    + \underbrace{\frac{(T-t+2)}{2}\cdot\frac{1}{h_i}\sum_{l=N_{i,t-1}-h_i+1}^{N_{i,t-1}}\frac{X_i(l) - X_i(l-h_i)}{h_i}}_{\text{(ii) Expected future {\it cumulative} gain }}  \notag \\
   &+ \underbrace{\sigma\sqrt{\frac{2\left[3(T-t+2)^2+8h_i^2\right]\log t^3}{4h_i^3}}}_{\text{(iii) Exploration bonus}} 
\end{align}
where $\sigma^2$ is the known variance parameter and $h_i:=\max\left\{\lfloor \epsilon N_{i,t-1} \rfloor,1\right\}$.
The sliding window parameter $\epsilon$ controls a bias-variance tradeoff in estimation.

The components of $B_i(t)$ are interpreted as follows:
\begin{enumerate} [label=(\roman*)]
    \item {\it{Recent average}} estimates the current expected reward of arm $i$.
    \item {\it{Expected future cumulative gain}} captures the future potential from the arm's growth over the remaining horizon $(T-t+1)$.
    \item {\it{Exploration bonus}} accounts for estimation uncertainty, following the UCB principle \citep{lattimore2020bandit}.
\end{enumerate}

\textbf{Remark: Comparison between CURE-UCB and R-ed-UCB}
\label{rem:intuition}

To clarify the benefit of horizon awareness in the index design, it is instructive to contrast CURE-UCB with R-ed-UCB \citep{metelli2022stochastic}, which is a representative horizon-agnostic rising algorithm.
Among existing baselines, R-ed-UCB is algorithmically closest to CURE-UCB: both are UCB-based algorithms and utilize estimates of reward growth rates. Therefore, comparing these two algorithms isolates the central effect of horizon awareness in the index.
The index of R-ed-UCB is given as follows:
\begin{align} \label{eq:red_index}
    B_{i}^{red}(t) := &\underbrace{\frac{1}{h_i}\sum_{l=N_{i,t-1}-h_i+1}^{N_{i,t-1}}X_i(l)}_{\text{(i) Recent average}} \notag 
    + \underbrace{\frac{1}{h_i}\sum_{l=N_{i,t-1}-h_i+1}^{N_{i,t-1}}(t-l)\frac{X_i(l) - X_i(l-h_i)}{h_i}}_{\text{(ii) Expected future gain ({\it agnostic of horizon} $T$)}} \\ 
    &+ \underbrace{\sigma(t-N_{i,t-1}+h_{i}-1)\sqrt{\frac{10\log(1/\delta_t )}{h_i^3}}}_{\text{(iii) Exploration bonus}} \;,
\end{align}
where $\delta_t$ is the confidence parameter.
While both algorithms share the same overall structure, the sums of terms (i) and (ii) estimate fundamentally different quantities.
Firstly, R-ed-UCB targets the expected reward at a single future point $\mu_i(t)$:
\begin{align} 
    \mu_i(t) &= \mu_i(n) + \sum_{l=n}^{t-1}\gamma_i(l) \leq \mu_i(n) + (t-n)\gamma_i(n) \;, \label{eq:red_1}
\end{align}
where \eqref{eq:red_1} follows from Assumption~\ref{asu:concavity}.
Plugging in sample estimates of $\mu_i(n)$ and $\gamma_i(n)$ and adding an exploration bonus yields the R-ed-UCB index given in \eqref{eq:red_index}.
We note that $\mu_i(t)$ is a surrogate for the optimality criterion of Proposition~\ref{prop:optimal_structure}, adopted in the absence of horizon knowledge. This surrogate is sound as $T \to \infty$, where the optimal policy commits to the arm with the highest asymptotic reward. At a finite $T$, however, it is not exactly aligned with the optimality criterion, which can lead to a suboptimal choice in cases such as Figure~\ref{fig:intro}.

In contrast, CURE-UCB targets the cumulative reward over the remaining horizon:
\begin{align}
    \sum_{l=1}^{T-t+1} \mu_i(N_{i,t-1}+l) &\leq \sum_{l=1}^{T-t+1} \left(\mu_i(N_{i,t-1}) + l\gamma_i(N_{i,t-1})\right) \label{eq:cure_1}\\
    &= (T-t+1)\mu_i(N_{i,t-1}) + \frac{(T-t+1)(T-t+2)}{2}\gamma_{i}(N_{i,t-1}) \;,
\end{align}
where \eqref{eq:cure_1} follows from Assumption~\ref{asu:concavity}.
Dividing by $(T-t+1)$ and plugging in sample estimates with an exploration bonus recovers the CURE-UCB index (\eqref{eq:cure_index}).
This targets the remaining-horizon cumulative reward that governs optimality in Proposition~\ref{prop:optimal_structure}, via a concave upper bound that knowledge of $T$ makes available.

\section{Analysis} \label{sec:analysis}
In this section, we provide several theoretical analyses of the proposed CURE-UCB algorithm. 
We denote CURE-UCB with sliding-window parameter $\epsilon$ by $\pi_{\epsilon}^{\mathrm{cure}}$.
We first analyze CURE-UCB in the deterministic LTF setting (Section~\ref{sec:det_ltf}) and then derive a regret bound for general concave rising rewards (Section~\ref{sec:regret_analysis}).


\subsection{Horizon Awareness in the Linear-Then-Flat (LTF) setting}
\label{sec:det_ltf}

To substantiate the theoretical advantage of horizon awareness, we analyze the performance in a concrete structured environment, i.e., the ``Linear-Then-Flat'' (LTF) setting, which serves as a canonical example where the optimal strategy shifts sharply based on the horizon.
We note that this LTF setting is an illustrative example: a simplified abstraction of the increasing-then-converging training curves observed in practice \citep{xia2024llm,li2018hyperband}.
Our regression analysis in Appendix~\ref{sec:ltf_regression}
further shows that LTF closely approximates these real-world
training curves.
We first formally define the LTF instance class:
\begin{definition}(Linear-Then-Flat Instance)
\label{def:ltf}
Let $\mathfrak{U}_{K}$ denote the set of $K$-armed Linear-Then-Flat problem instances. For any instance $\vec{\mu} \in \mathfrak{U}_{K}$, the expected reward of arm $i$ at its $n$-th play is:
\begin{align}
\mu_i(n) = \min\left\{b_i, a_i\cdot n \right\} \;,
\end{align}
where $a_i \in (0,1)$ is the growth rate (slope) and $b_i\in (0,1]$ is the maximum attainable reward. 
We define the saturation time of arm $i$ as $c_i := b_i/a_i$ and assume that $c_i \in \mathbb{N}$. 
Thus, arm $i$ grows linearly at rate $a_i$ until it reaches $b_i$ at $n=c_i$, after which it yields the constant reward $b_i$.
\end{definition}
For every deterministic LTF instance $\vec{\mu}\in\mathfrak{U}_K$, we establish the uniform dominance of CURE-UCB over R-ed-UCB \citep{metelli2022stochastic}.
For comparison, let $\pi^{\mathrm{red}}_{\epsilon}$ denote R-ed-UCB with sliding-window parameter $\epsilon$. 
Since each observed reward equals its exact mean, we can use a unit sliding window by setting $\epsilon=0$ and select every arm  only once during initialization.
\begin{theorem}[Uniform Dominance in deterministic LTF]
\label{thm:dominance_ltf}
Consider a deterministic LTF instance $\mu \in \mathcal{U}_K$ ($\sigma = 0$) and  $\pi_{0}^{\mathrm{cure}}$ and $\pi_{0}^{\mathrm{red}}$ with sliding-window parameter $\epsilon\!=\!0$.
Then, for any finite horizon $T$,
\begin{equation}
Reg_{\vec{\mu}}\!\left(\pi^{\mathrm{cure}}_{0}, T\right) \leq Reg_{\vec{\mu}}\!\left(\pi^{\mathrm{red}}_{0}, T\right)  \;.
\end{equation}
\end{theorem}


The intuition behind Theorem~\ref{thm:dominance_ltf} is that CURE-UCB evaluates whether further exploration is worthwhile based on the remaining time. Thus, when R-ed-UCB explores an additional growing arm despite insufficient time for that exploration to pay off, CURE-UCB avoids such unnecessary exploration, leading to the uniform regret dominance. The detailed proof is provided in Appendix~\ref{sec:proof_dominance_ltf}.


Furthermore, we show that the regret difference between CURE-UCB and R-ed-UCB can be critically large:
\begin{proposition}[The dominance gap can be linear]
\label{prop:linear_gap}
For every fixed $K\geq2$ and sufficiently large $T$, there exists a deterministic LTF instance
$\vec{\mu} \in \mathcal{U}_K$ such that:
\begin{equation}
    Reg_{\vec{\mu}}(\pi^{\mathrm{red}}_0, T) = \Omega(T), \quad Reg_{\vec{\mu}}(\pi^{\mathrm{cure}}_0, T) = O(1) \;.
\end{equation}
\end{proposition}
Proposition~\ref{prop:linear_gap} shows that the dominance in
Theorem~\ref{thm:dominance_ltf} can be far from marginal. 
This shows that the benefit of horizon awareness can shape performance at the leading order.
The detailed proof is provided in Appendix~\ref{sec:proof_linear_gap}.


\subsection{General Regret Upper Bound}
\label{sec:regret_analysis}
To quantify the complexity of general RMAB instances beyond the LTF setting, we adapt the definition of \textit{cumulative increment} from \citet{metelli2022stochastic}, which characterizes the total magnitude of reward growth:

\begin{definition} (Cumulative increment \citep{metelli2022stochastic})
    Let $\gamma_{\max}(n) := \max_{i\in [K]}\gamma_i(n)$ be the maximum instantaneous increment across all arms at the $n$-th play.
    For any $M\in[T]$ and a parameter $q\in[0,1]$, we define the \textit{cumulative increment} $\Upsilon_{\vec{\mu}}(M,q)$ as follows:
    \begin{align}
        \Upsilon_{\vec{\mu}}(M,q) := \sum_{n=1}^{M}(\gamma_{\max}(n))^q \;.
    \end{align}
\end{definition}

Leveraging this complexity measure, we establish the following regret upper bound, which guarantees the efficient performance of CURE-UCB in general concave rising environments.

\begin{theorem}[General regret upper bound]
\label{thm:upper_bound}
Consider a general rising concave instance $\vec{\mu}$ satisfies Assumption~\ref{asu:concavity} and $\pi_{\epsilon}^{\mathrm{cure}}$. 
For $T>K\left\lceil2/\epsilon\right\rceil, q\in [0,1]$, $\sigma>0$, and $\epsilon\in (0,0.5)$, we have :
\begin{align}
     Reg_{\vec{\mu}}(\pi^{\mathrm{cure}}_{\epsilon}, T) \leq  
     \frac{K(1+2^{-q})T^{q}}{1-2\epsilon}\Upsilon_{\vec{\mu}}\left(\left\lceil(1-2\epsilon)\frac{T}{K}\right\rceil,q\right) 
    &+ O\left(\frac{K}{\epsilon} (\sigma T)^{2/3} (\log T)^{1/3}\right)
\end{align}
\end{theorem}
This result provides a regret guarantee for CURE-UCB beyond the structured LTF setting, covering general concave rising rewards.
Moreover, the growth-dependent $\Upsilon_{\vec{\mu}}$-term is optimal in its dependence on $T$: \citet{fiandri2024rising} establish a lower bound of $\Omega(T^q\Upsilon_{\vec{\mu}}(T,q))$, which also applies to the known-horizon setting.
The proof of Theorem~\ref{thm:upper_bound} and the extension of the lower bound to the known-horizon setting are provided in Appendices ~\ref{sec:proof_upper_bound} and ~\ref{sec:lower_bound}, respectively.


\section{Experiments} \label{sec:experiments}


We evaluate CURE-UCB in two synthetic environments and two real-world model-selection tasks.
The synthetic experiments consider the LTF setting (Section~\ref{sec:exp_ltf}) and diverse concave reward dynamics (Section~\ref{sec:exp_rf}), while the real-world experiments cover online model selection on IMDB (Section~\ref{sec:IMDB}) and LLM selection (Section~\ref{sec:llm}).
While real training curves need not be exactly monotone or concave at every step, they exhibit the overall increasing-then-converging trend motivating the RMAB model. 
We use these tasks to evaluate CURE-UCB beyond the idealized assumptions used in our analysis.
Detailed experimental specifications are provided in Appendix~\ref{sec:explantion_experiment}, and additional robustness analysis on window size and reward noise are reported in Appendix~\ref{sec:robustness}.

\paragraph{Baselines} 
We compare CURE-UCB against two categories of algorithms. Detailed descriptions for all baseline algorithms are provided in Appendix~\ref{sec:baselines}.
\vspace{-0.2cm}
\begin{itemize}[left=0.1cm]
    \item \textbf{Rising bandit algorithms}: We adopt state-of-the art algorithms designed for the rising bandit framework, including R-ed-UCB \citep{metelli2022stochastic}, ET-SWGTS \citep{fiandri2025thompson}, TI-UCB \citep{xia2024llm}, and R-ed-AE \citep{amichay2025rising}. Among these, R-ed-AE is horizon-aware, whereas the others are horizon-agnostic.
    \item \textbf{Non-stationary bandit algorithms}: As the RMAB framework is a subclass of non-stationary environments, we include representative algorithms handling  non-stationary rewards. These comprise sliding-window approaches (SW-UCB \citep{garivier2011upper}, SW-KL-UCB \citep{garivier2011kl}, SW-TS \citep{trovo2020sliding}), as well as the restarting-based Rexp3 \citep{besbes2014stochastic} and the sequential elimination-based Ser4 \citep{allesiardo2017non}.
\end{itemize}

\begin{figure*}[t]
     \vspace{-0.5cm}
    \centering
    \begin{minipage}[c]{0.62\textwidth}
        \centering
        \includegraphics[width=\linewidth]{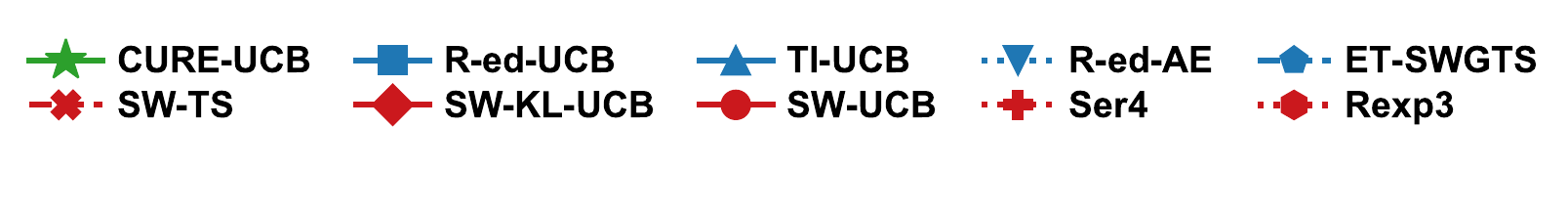}
    \end{minipage}
    \hfill
    \begin{minipage}[c]{0.36\textwidth}
        \centering
        \includegraphics[width=0.9\linewidth]{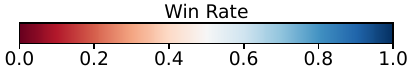}
    \end{minipage}

    \vspace{-0.1cm}

    \begin{subfigure}[b]{0.36\textwidth}
        \centering
        \includegraphics[width=\linewidth, trim={7bp 10bp 8bp 3bp}, clip]{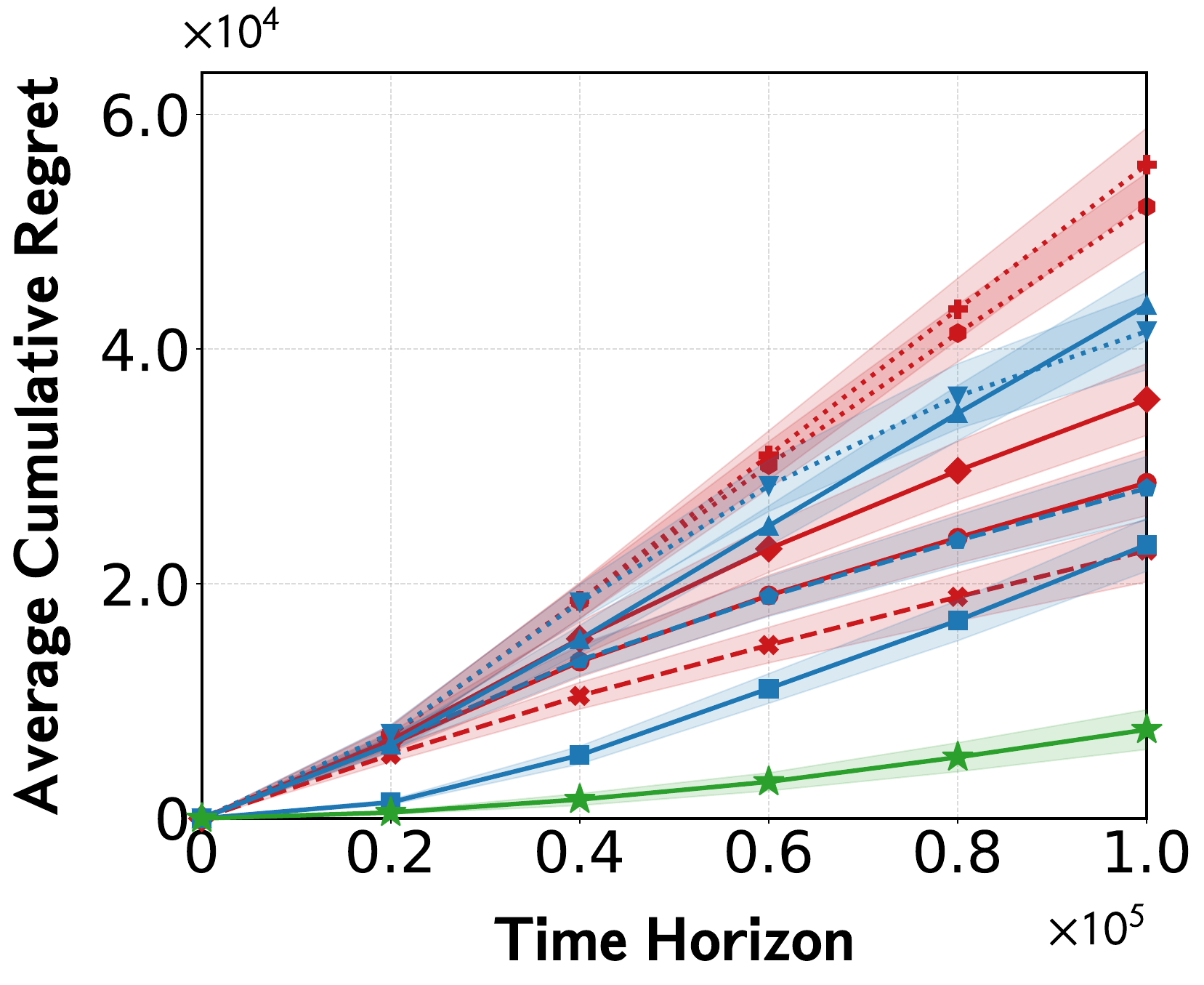}
        \caption{\footnotesize Cumulative Regret}
        \label{fig:ltf_regret}
    \end{subfigure}
    \hfill
    \begin{subfigure}[b]{0.31\textwidth}
        \centering
        \includegraphics[width=\linewidth, trim={5bp 5bp 5bp 5bp}, clip]{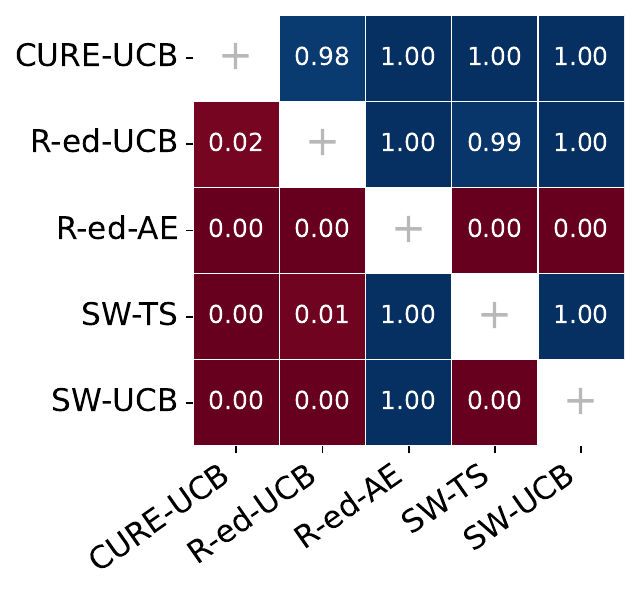}
        \caption{\footnotesize Pairwise win rate ($T\!=\!20K$)}
        \label{fig:ltf_winrate_20k}
    \end{subfigure}
    \hfill
    \begin{subfigure}[b]{0.31\textwidth}
        \centering
        \includegraphics[width=\linewidth, trim={5bp 5bp 5bp 5bp}, clip]{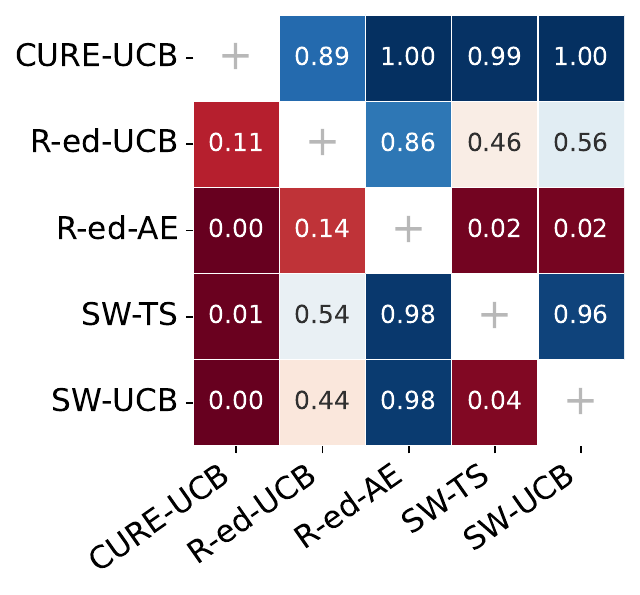}
        \caption{\footnotesize Pairwise win rate ($T\!=\!100K$)}
        \label{fig:ltf_winrate_100k}
    \end{subfigure}

    \vspace{0.2cm}

    \begin{subfigure}[b]{0.36\textwidth}
        \centering
        \includegraphics[width=\linewidth, trim={7bp 10bp 8bp 3bp}, clip]{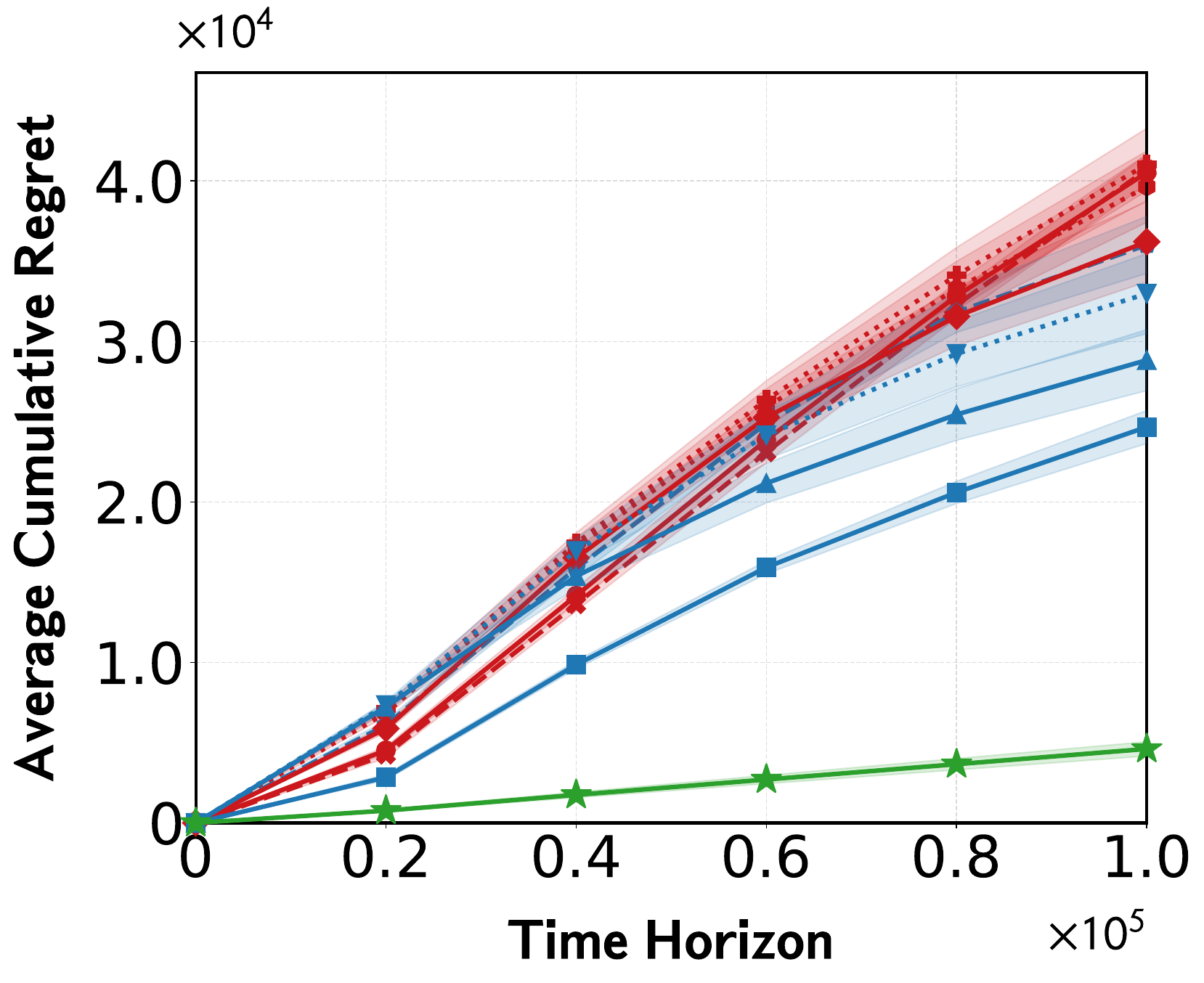}
        \caption{\footnotesize Cumulative Regret}
        \label{fig:concave_regret}
    \end{subfigure}
    \hfill
    \begin{subfigure}[b]{0.31\textwidth}
        \centering
        \includegraphics[width=\linewidth, trim={5bp 5bp 5bp 5bp}, clip]{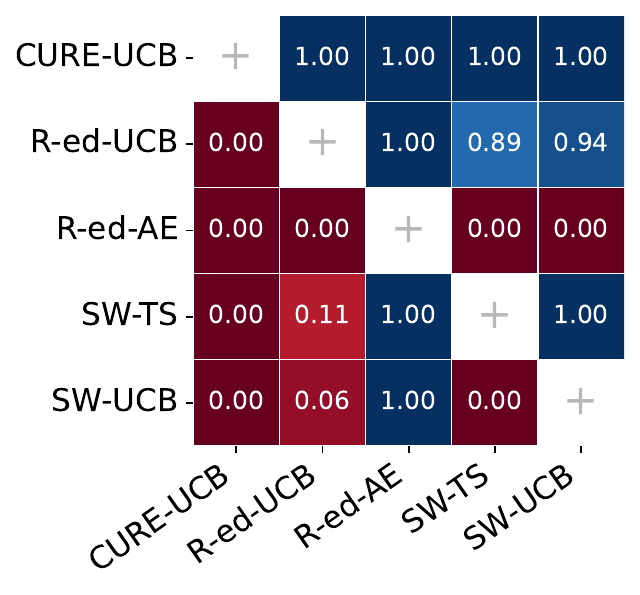}
        \caption{\footnotesize Pairwise win rate ($T\!=\!20K$)}
        \label{fig:concave_winrate_20k}
    \end{subfigure}
    \hfill
    \begin{subfigure}[b]{0.31\textwidth}
        \centering
        \includegraphics[width=\linewidth, trim={5bp 5bp 5bp 5bp}, clip]{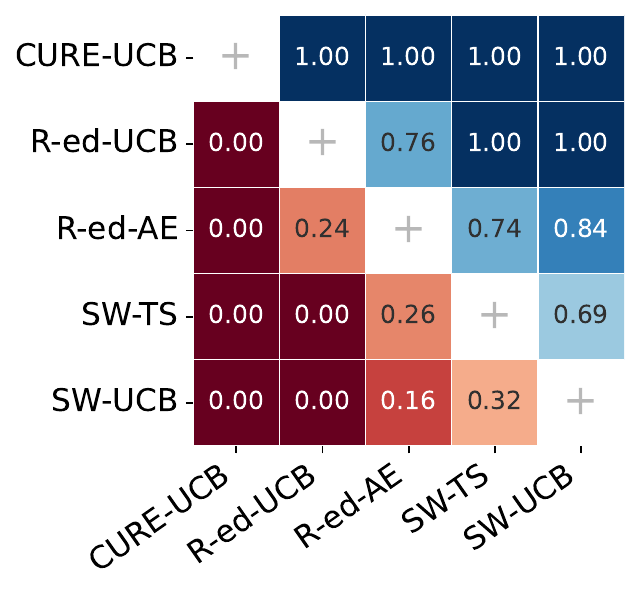}
        \caption{\footnotesize Pairwise win rate ($T\!=\!100K$)}
        \label{fig:concave_winrate_100k}
    \end{subfigure}

    \caption{\textbf{Performance Analysis in LTF (top row) and Concave (bottom row) Settings.} (a, d) Cumulative regret as a function of the time horizon $T$. (b, e) Pairwise win rate at $T\!=\!20K$ (short horizon). (c, f) Pairwise win rate at $T\!=\!100K$ (long horizon). Each cell of the win-rate matrices shows how often the row algorithm attains lower regret than the column algorithm. 
    Shaded regions and error bars denote 95\% confidence intervals over 40 seeds, with the full results over all baselines in Appendix~\ref{sec:additional_experiments}.}
    \label{fig:synthetic_all_results}
    \vspace{-0.3cm}
\end{figure*}

\paragraph{Evaluation Metrics}
We report final cumulative regret across separately evaluated
horizons $T$.
We also report the pairwise win rate of the representative algorithms, calculated over 100 independent problem instances at fixed horizons ($T=20K$ and $100K$), where each instance is run with 40 random seeds.

\subsection{Linear-Then-Flat Setting} 
\label{sec:exp_ltf}

To construct a rigorous evaluation environment for the Linear-Then-Flat (LTF) setting, we generated problem instances strictly following the formulation in Definition~\ref{def:ltf}.
We enforce structural diversity by randomizing the parameters rather than fixing them  (detailed specifications in Appendix~\ref{sec:explantion_experiment}).
CURE-UCB performs well across the LTF instances: it achieves lower regret than every baseline on most instances at both horizons (Figures~\ref{fig:ltf_winrate_20k} and~\subref{fig:ltf_winrate_100k}) and attains the lowest cumulative regret overall (Figure~\ref{fig:ltf_regret}).
At the long horizon, CURE-UCB achieves a pairwise win rate
of 89\% against R-ed-UCB, demonstrating a substantial empirical
advantage under stochastic LTF rewards.


\subsection{Concave Setting}
\label{sec:exp_rf}

We next extend our evaluation to the concave setting, constructed using diverse concave functions to simulate diverse non-linear dynamics (detailed specifications in Appendix~\ref{sec:explantion_experiment}).
Each instance places an early peaker against a late bloomer, so that which arm is best depends on the horizon $T$.
As shown in Figure~\ref{fig:concave_regret}, CURE-UCB attains the lowest average cumulative regret across all horizons, and it achieves lower regret than every baseline on essentially all instances at both the short ($T=20K$) and long ($T=100K$) horizons (Figures~\ref{fig:concave_winrate_20k} and~\subref{fig:concave_winrate_100k}).
This consistency across horizons highlights robustness of CURE-UCB.

\subsection{Online Model Selection (IMDB)} \label{sec:IMDB}
To demonstrate the practical applicability of our framework, we investigate the problem of online model selection for sentiment analysis using the IMDB dataset \citep{maas2011learning}. 
In this task, each arm represents a distinct learning model (e.g., neural networks with different hyperparameters).
The objective is to maximize cumulative validation accuracy through sequential model selection within a finite training budget $T$.
We follow the IMDB online model-selection setup of ~\citet{metelli2022stochastic}, including the candidate models and training protocol.
As illustrated in Figure~\ref{fig:imdb_reward}, real-world training curves exhibit persistent fluctuations and noise, presenting a significantly more challenging environment compared to synthetic settings. 

\subsection{LLM selection}
\label{sec:llm}
We further evaluate CURE-UCB on the LLM fine-tuning task introduced by~\citet{xia2024llm},  where each arm corresponds to a candidate language model and the per-step reward reflects model performance penalized by accumulated fine-tuning cost.
We use five arms: four Qwen3.5 \citep{qwen35blog}  models (2B, 4B, 9B, 27B), each fine-tuned with LoRA; their per-step reward rises as training proceeds and is penalized by an accumulated, size-dependent cost, so larger models improve more but pay a higher price. The fifth arm is Qwen3.8-Max \citep{qwen38}, a 2.4T-parameter mixture-of-experts model with 95B active parameters. 
At this scale fine-tuning is practically infeasible, so we use Qwen3.8-Max off-the-shelf with in-context examples only. Because it is never fine-tuned, its reward stays high but flat, whereas the smaller fine-tuned arms start lower and rise with training. Detailed setup is in Appendix~\ref{sec:llm_setting}.
As shown in Figure~\ref{fig:llm_regret}, CURE-UCB consistently achieves the lowest cumulative regret across all horizons, outperforming other baselines. 
This result demonstrates the robustness of CURE-UCB across diverse horizons in LLM-scale model selection.


\begin{figure}[t]
    \vspace{-0.5cm}
    \centering
    \captionsetup[subfigure]{
        labelfont=normalfont,
        textfont=normalfont,
        justification=centering,
        singlelinecheck=true,
        skip=2bp
    }

    \begin{minipage}{0.94\linewidth}
        \centering
        \setlength{\parskip}{0pt}

        \begin{subfigure}[b]{0.48\linewidth}
            \centering
            \includegraphics[width=\linewidth]{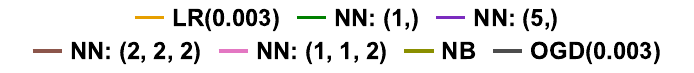}\par
            \nointerlineskip\vspace{1bp}
            \includegraphics[width=\linewidth]{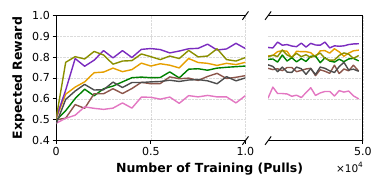}
            \caption{IMDB reward dynamics}
            \label{fig:imdb_reward}
        \end{subfigure}%
        \hfill
        \begin{subfigure}[b]{0.48\linewidth}
            \centering
            \includegraphics[width=\linewidth]{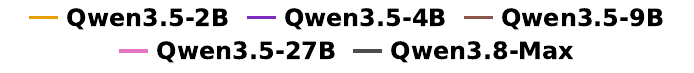}\par
            \nointerlineskip\vspace{1bp}
            \includegraphics[width=\linewidth]{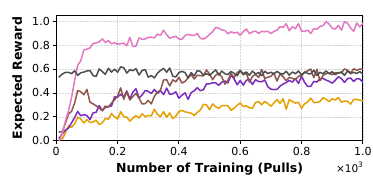}
            \caption{XSum-LLM reward dynamics}
            \label{fig:llm_reward}
        \end{subfigure}

        \par\nointerlineskip\vspace{3bp}
        \includegraphics[width=\linewidth]{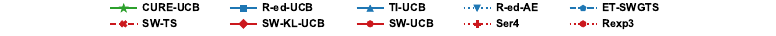}
        \par\nointerlineskip\vspace{2bp}

        \begin{subfigure}[b]{0.48\linewidth}
            \centering
            \includegraphics[width=\linewidth]{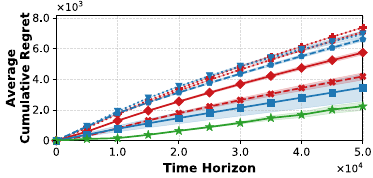}
            \caption{IMDB cumulative regret}
            \label{fig:imdb_regret}
        \end{subfigure}%
        \hfill
        \begin{subfigure}[b]{0.48\linewidth}
            \centering
            \includegraphics[width=\linewidth]{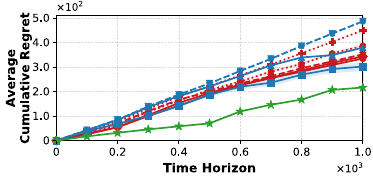}
            \caption{XSum-LLM cumulative regret}
            \label{fig:llm_regret}
        \end{subfigure}
    \end{minipage}

    \caption{
    \textbf{Online model selection on IMDB and XSum-LLM.}
    The top row shows reward dynamics, and the bottom row shows cumulative regret.
    The left column corresponds to IMDB sentiment analysis for horizons up to $T\!=\!50K$,
    and the right column corresponds to LLM fine-tuning on XSum summarization for
    horizons up to $T\!=\!1K$.
    The legends above (a) and (b) indicate model configurations, and the shared legend
    between the two rows indicates the algorithms used in the cumulative-regret plots.
    Shaded regions denote 95\% confidence intervals over 40 seeds.
    }
    \vspace{-0.3cm}
    \label{fig:realworld_combined}
\end{figure}

\section{Conclusion} \label{sec:conclusion}
In this paper, we highlighted the role of horizon awareness in RMAB, showing that knowledge of the horizon $T$ enables direct alignment with the optimality structure of RMAB. As a concrete instantiation of this principle, we proposed CURE-UCB and established uniform dominance over R-ed-UCB in the deterministic LTF settings, a regret upper bound for general concave rising rewards, and consistent empirical gains across synthetic and various real-world tasks such as LLM model selection.
One limitation of our work is that the uniform dominance result is restricted to the deterministic LTF setting. 
Extending this guarantee to broader stochastic or non-LTF settings is a natural direction for future work.


\section*{AI Use Statement}
In this work, we used generative AI tools to implement algorithms and experiments, assist with translation, provide critical ingredients for proofs, and check the correctness of the proofs of our theoretical results. We did not use generative AI tools to develop the problem formulation or theoretical framework, formulate mathematical claims, design the research methodology or experiments, generate synthetic datasets, clean or reformat datasets, or interpret results; qualitative and thematic data analysis is not applicable to this work. Additionally, we used generative AI tools to draft parts of the paper and edit the text for readability. We reviewed all AI-assisted work. All theorems and proofs were formulated and written by the authors; LLM-based checking served only as an additional verification step, and every issue it raised was examined and resolved by the authors. LLM-generated code was reviewed and tested by the authors for consistency with the algorithm descriptions in the paper. All AI-assisted text was revised by the authors, and every technical statement was checked against the corresponding results. We take responsibility for the final content of this work, including text, claims, and artifacts produced with the aid of generative AI.

\bibliography{iclr2027_conference}
\bibliographystyle{iclr2027_conference}

\appendix
\section{Appendix}
This material provides proof of theorems, details of environments and baselines, and additional experimental results:
\begin{itemize}[left = 0.3cm]
\item \textbf{Appendix~\ref{sec:explantion_experiment}:} Detailed explanation for experiment
\item \textbf{Appendix~\ref{sec:ltf_regression}:}
Additional analysis for LTF setting
\item \textbf{Appendix~\ref{sec:proof_dominance_ltf}:} Proof of Theorem~\ref{thm:dominance_ltf}
\item \textbf{Appendix~\ref{sec:proof_linear_gap}:} Proof of Proposition~\ref{prop:linear_gap}
\item \textbf{Appendix~\ref{sec:proof_upper_bound}:} Proof of Theorem~\ref{thm:upper_bound}
\item \textbf{Appendix~\ref{sec:lower_bound}:} Known-horizon extension of Thm 4.6 in \citep{fiandri2024rising}
\item \textbf{Appendix~\ref{sec:proof_lemma}:} Proof of Lemmas
\item \textbf{Appendix~\ref{sec:baselines}:} Detailed description of the baselines in Section~\ref{sec:experiments}.
\item
\textbf{Appendix~\ref{sec:additional_experiments}:}
Additional experimental results
\item
\textbf{Appendix~\ref{sec:robustness}:}
Additional experiments for robustness
\end{itemize}

\section{Detailed explanation for experiment}
\label{sec:explantion_experiment}

\subsection{Synthetic Environment Generation}
In this section, we explain how we randomly generate the problem instances.
\subsubsection{Linear-Then-Flat (LTF) Setting}
We generated 100 independent problem instances for the LTF setting. Each instance consists of $K \in \{2, 3, 4, 5\}$ arms, where each arm's reward function follows the form $\mu_i(n) = \min(b_i, a_i \cdot n)$. The saturation level $b_i$ is sampled uniformly from $[0.1, 1.0]$, and the saturation time $t_{\text{sat}}$ is randomly selected from $[0.05T, T]$ for $T=100K$. The slope is then determined as $a = b / t_{\text{sat}}$.
\begin{figure*}[h]
    \centering
    \begin{subfigure}[b]{0.30\textwidth}
        \centering
        \includegraphics[width=\linewidth, trim={7bp 10bp 8bp 3bp}, clip]{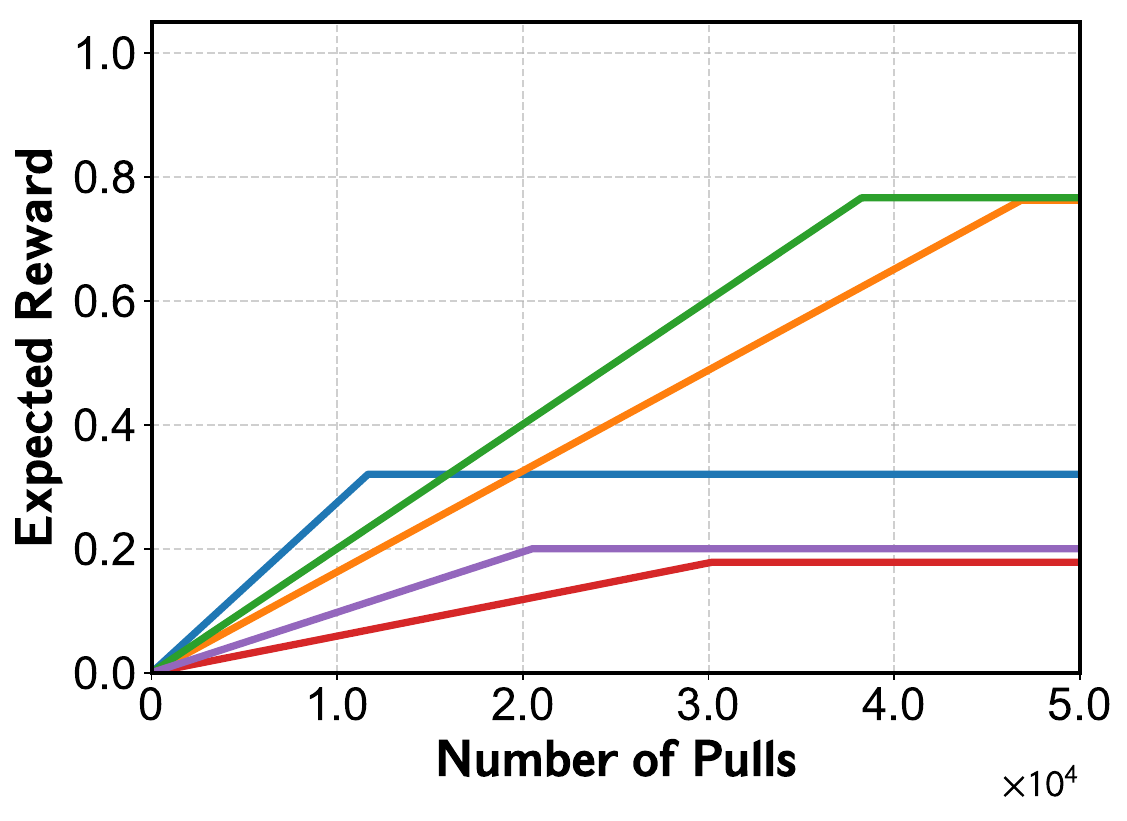}
        
        \caption{\small Example 1} 
        
        \label{fig:ltf_ex1}
    \end{subfigure}
    \hfill
    \begin{subfigure}[b]{0.30\textwidth}
        \centering
        \includegraphics[width=\linewidth, trim={5bp 5bp 5bp 5bp}, clip]{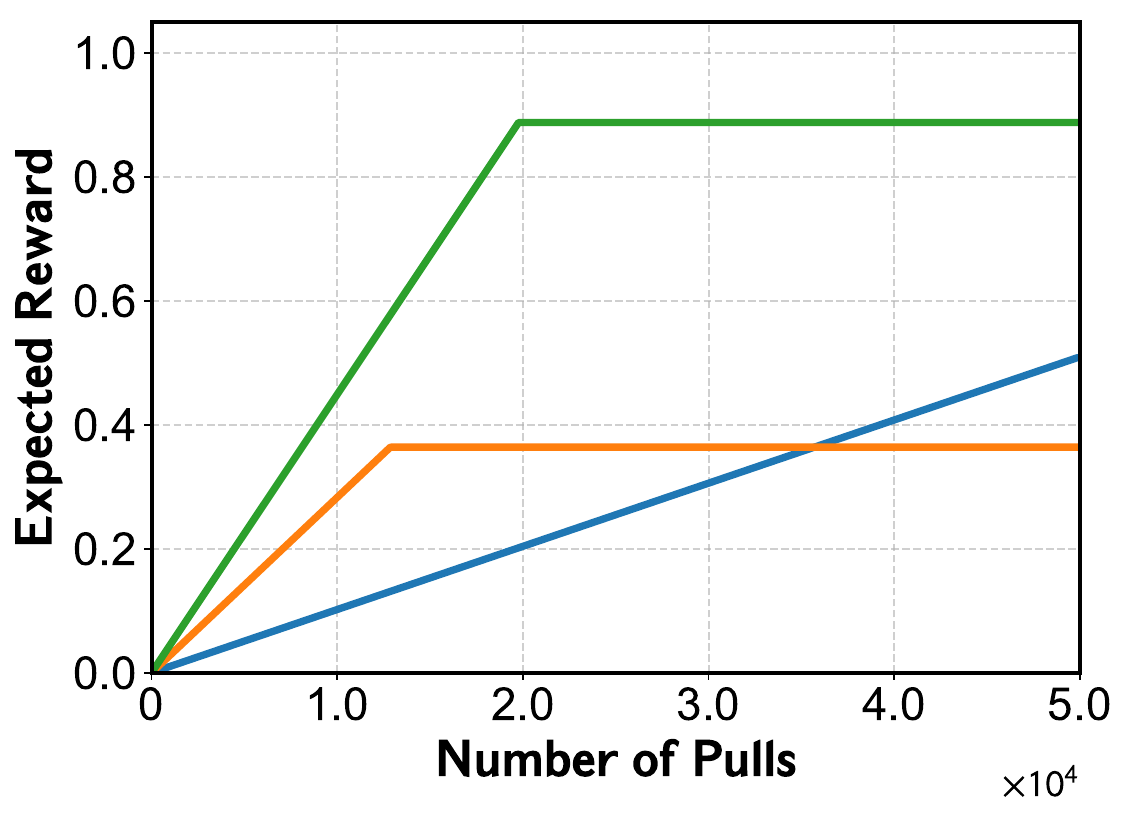}
        \caption{\small Example 2}
        \label{fig:ltf_ex2}
    \end{subfigure}
    \hfill
    \begin{subfigure}[b]{0.30\textwidth}
        \centering
        \includegraphics[width=\linewidth, trim={5bp 5bp 5bp 5bp}, clip]{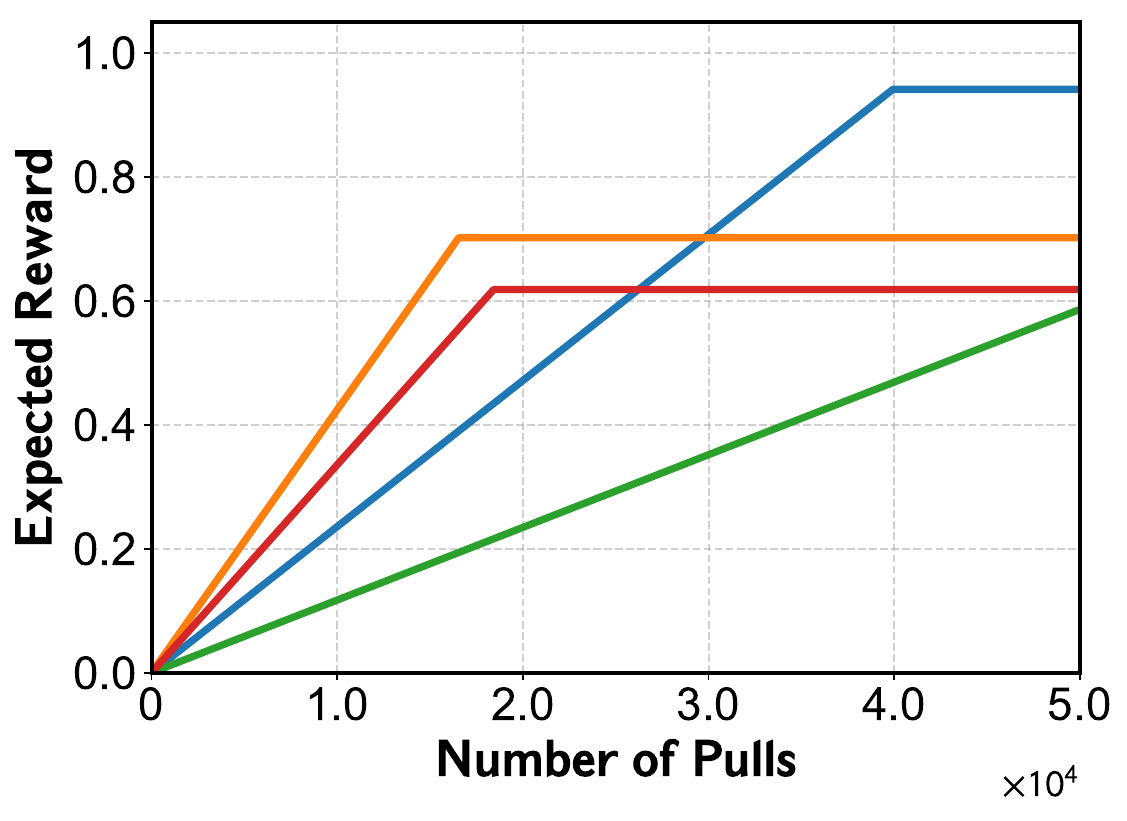}
        \caption{\small Example 3}
        \label{fig:ltf_ex3}
    \end{subfigure}
    
     \caption{\textbf{Example instances of LTF Setting.} }
    \label{fig:ltf_ex}
\end{figure*}

\subsubsection{Concave Setting}
We generate 100 independent problem instances for the concave setting. The number
of arms in each instance is sampled uniformly from $K \in \{3, \ldots, 15\}$, and
the arm order is randomly permuted within each instance so that the optimal arm is
not identified by its index. 
The purpose of this synthetic benchmark is not to model the empirical distribution of real training curves, but to create diverse and nontrivial concave instances in which the preferred arm can genuinely depend on the horizon. Fully unconstrained random generation often produces trivial instances in which one arm dominates across all horizons, making horizon awareness largely irrelevant. We therefore use a role-based construction with an early peaker, a late bloomer, near-optimal challengers, and distractors, while randomizing the number of arms, arm ordering, growth families, and curve parameters to preserve substantial instance diversity.
Each instance contains four types of arms: one optimal
late bloomer, one early peaker, $\lfloor K/2 \rfloor$ near-optimal challengers, and
$K - 2 - \lfloor K/2 \rfloor$ distractors. The optimal late bloomer is a low-start,
high-potential; the early peaker starts high but saturates low, making it attractive
early and inferior later; the near-optimal challengers closely track the optimal late
bloomer but stay strictly below it; and the distractors saturate at clearly lower
levels.

The expected reward of an arm at its $n$-th pull is a bounded concave curve
\begin{align*}
\mu(n) = s + (L - s)\, g(n),
\end{align*}
where $s = \mu(0)$ is the initial reward, $L = \mu(H)$ is the reward reached at the
design horizon $H = 100{,}000$, and $g(n) \in [0, 1]$ is a normalized growth
function with $g(0) = 0$ and $g(H) = 1$. For each arm, we randomly draw $g$ from three families:
\begin{align*}
g_{\mathrm{rat}}(n; k)  &= \min\!\left\{1, \frac{n/(n+k)}{H/(H+k)}\right\}, \\[4pt]
g_{\exp}(n; k)  &= \min\!\left\{1, \frac{1 - e^{-kn}}{1 - e^{-kH}}\right\}, \\[4pt]
g_{\mathrm{atan}}(n; k) &= \min\!\left\{1, \frac{\arctan(kn)}{\arctan(kH)}\right\}.
\end{align*}
each normalized to reach one at $n = H$ and held constant afterward by the
$\min\{1, \cdot\}$ cap.

Each arm role is specified by its initial reward $s$ and its horizon value
$L = \mu(H)$. The optimal late bloomer has the lowest start and highest ceiling,
$s = 0.10$ and $L = 0.95$, using the exponential family with $k = 4/30{,}000$ so
that it completes most of its growth well before $H$. The early peaker has
$s \in [0.40, 0.58]$ and $L \leq 0.60$. Each near-optimal challenger has
$s \sim \mathcal{U}[0.08, 0.12]$ and $L \sim \mathcal{U}[0.85, 0.90]$, and its
growth blends the optimal late bloomer's fast exponential shape with a slower
sampled curve, so that it shadows the optimal late bloomer while keeping a small
tail that improves up to $H$. Each distractor has $s \sim \mathcal{U}[0.05, 0.50]$
and $L \leq 0.70$.

Although every normalized curve formally reaches its limit $L$ only at $n = H$, the
shape parameter $k$ governs how early an arm effectively saturates, i.e., how quickly
it approaches $L$. We control this per arm through a target checkpoint $n^{\star}_i$:
for each arm $i$, we set $k$ so that about $85\%$ of the rise from
$s$ to $L$ is completed by $n^{\star}_i$, solving for $k$ in closed form for the
rational family and by binary search for the exponential and arctangent families,
with a uniform $\pm 5\%$ multiplicative perturbation applied to the solved $k$. The
optimal late bloomer instead uses a directly fixed rate $k = 4/30{,}000$, which
front-loads its growth so that it rises faster than its competitors.

\begin{figure*}[h]
    \centering
    \begin{subfigure}[b]{0.30\textwidth}
        \centering
        \includegraphics[width=\linewidth, trim={5bp 5bp 5bp 3bp}, clip]{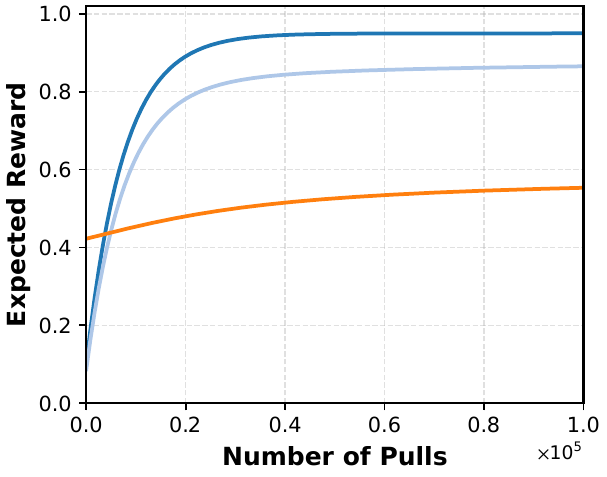}
        
        \caption{\small Example 1} 
        
        \label{fig:concave_ex1}
    \end{subfigure}
    \hfill
    \begin{subfigure}[b]{0.30\textwidth}
        \centering
        \includegraphics[width=\linewidth, trim={5bp 5bp 5bp 3bp}, clip]{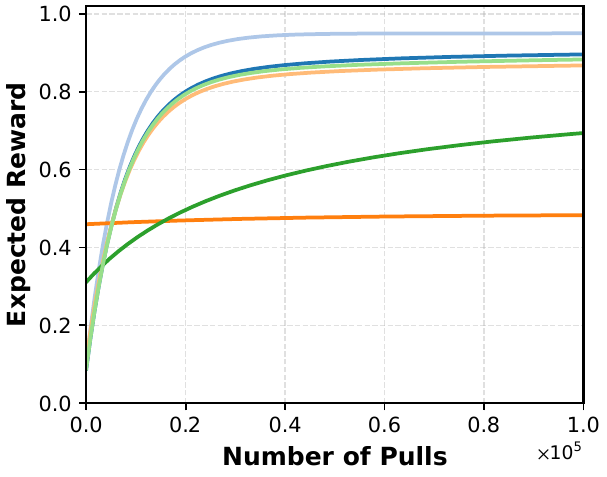}
        \caption{\small Example 2}
        \label{fig:concave_ex2}
    \end{subfigure}
    \hfill
    \begin{subfigure}[b]{0.30\textwidth}
        \centering
        \includegraphics[width=\linewidth, trim={5bp 5bp 5bp 3bp}, clip]{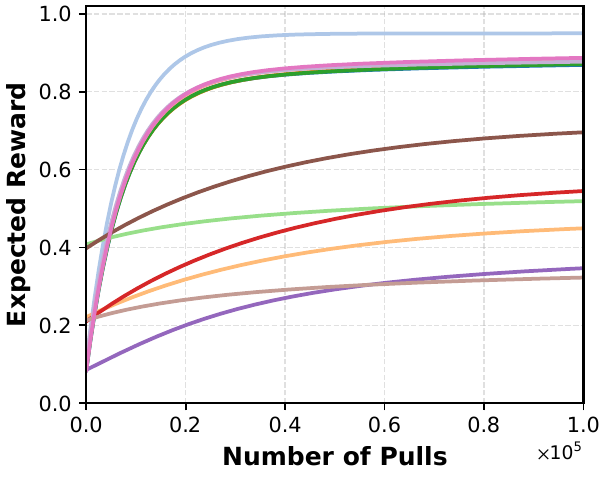}
        \caption{\small Example 3}
        \label{fig:concave_ex3}
    \end{subfigure}
    
     \caption{\textbf{Example instances of Concave Setting.} }
    \label{fig:concave_ex}
\end{figure*}

\subsection{LLM model selection}
\label{sec:llm_setting}

We largely follow the LLM fine-tuning setup of \citet{xia2024llm}, 
adapting their cost-penalized reward formulation to the Qwen3.5 model family and Qwen 3.8 max. To enable bandit simulation without re-training LLMs at every 
trial, we collect reward trajectories offline and replay them 
stochastically during simulation.

\paragraph{Arms and fine-tuning setup.}
We use five arms. Four correspond to the Qwen3.5 \cite{qwen35blog} family at sizes 2B, 
4B, 9B, and 27B, each fine-tuned independently for abstractive 
summarization on XSum (sub-sampled to 16000 train / 1000 eval 
documents with seed 42). We use a plain-text prompt template
\begin{quote}
\texttt{Summarize the following document:\textbackslash n\textbackslash n\{document\}\textbackslash n\textbackslash nSummary:}
\end{quote}
without chat formatting or reasoning mode. Each arm is adapted with LoRA 
($r{=}16$, $\alpha{=}32$, dropout $0.05$) on all linear projections trained with AdamW (learning rate 
$5{\times}10^{-5}$, weight decay $0.01$, cosine schedule with $100$ 
warm-up steps), gradient accumulation to a constant effective batch size 
of $16$, mixed-precision bfloat16, and a fixed budget of $1000$ 
optimizer steps. ROUGE-2 F1 is evaluated on a 200-document held-out 
subset every $10$ steps, giving one checkpoint per evaluation. The 27B arm additionally uses QLoRA with NF4 
4-bit quantization on four 48 GB GPUs. The fifth arm, Qwen3.8-Max, is not fine-tuned; we evaluate it off-the-shelf with two in-context examples on the same subset, yielding a constant ROUGE-2 score across checkpoints.

\paragraph{Reward construction.}
The raw ROUGE-2 checkpoint trajectories $\{\rho^{(k)}_t\}$ are converted into 
a fixed reward function via the following steps:
\begin{enumerate}
    \item \textbf{Smoothing:} a 5-checkpoint uniform moving average yields 
    $\tilde{\rho}^{(k)}_t$. This suppresses checkpoint-to-checkpoint evaluation noise so that the underlying learning trend, rather than measurement jitter, drives the reward.
    \item \textbf{Cost penalty:} a size-dependent penalty gives the 
    net reward $\nu^{(k)}_t = \tilde{\rho}^{(k)}_t - m_k \cdot t$ with 
    cost coefficients $m_k \in \{5{\times}10^{-5},\, 7.5{\times}10^{-5},\, 
    1.05{\times}10^{-4},\, 1.6{\times}10^{-4}\}$ for $k = 2$B, 4B, 9B, 27B. This charges each arm for the compute it spends, so that larger models, which improve more per step but are costlier to train, face a proportionally higher price. Since Qwen3.8-Max is not fine-tuned, it incurs no accumulated fine-tuning cost, so $m_k = 0$ for this arm.
    \item \textbf{Peak-capping:} we enforce monotonicity by setting 
    $\nu^{(k)}_t \leftarrow \nu^{(k)}_{t^\star_k}$ for all 
    $t \ge t^\star_k = \arg\max_t \nu^{(k)}_t$, satisfying the 
    rising-bandit assumption. We note that this is not just assumption to satisfy rising assumption. Instead, it indicates stopping the training after sufficient learning in real-world scenarios to avoid overfitting.
    \item \textbf{Perturbation and normalization:} we add a fixed Gaussian 
    perturbation $\eta^{(k)}_t \sim \mathcal{N}(0, 0.005^2)$, drawn once and held fixed across all simulations, 
    zero-clip negative values, and apply a global affine 
    rescaling 
    $r^{(k)}_t = (\nu'^{(k)}_t - \min)/(\max - \min)$, 
    where $\min$ and $\max$ are taken over all $(k, t)$ pairs in the 
    perturbed net-reward arrays. The perturbation keeps the curves from being perfectly smooth, mimicking the small run-to-run variation of real training, while the shared rescaling maps all arms onto a common $[0,1]$ scale without distorting their relative order. This yields the fixed reward function $\mu_i$.
\end{enumerate}

\paragraph{Simulation.}
The environment is rested: $\mu_i$ is evaluated at each arm's local pull count, so at round $t$ the selected arm returns $\mu_i(N_{i,t-1})$, obtained by linear interpolation 
(\texttt{np.interp}) over the precomputed length-$1000$ array, holding the first and last checkpoint values beyond the trajectory range. The algorithm observes $Y_{i,t} = \mu_i(N_{i,t-1}) + \varepsilon_{i,t}$ with $\varepsilon_{i,t} \sim \mathcal{N}(0, \sigma_{\mathrm{obs}}^2)$, so it queries an offline-measured reward rather than re-running the model at every trial.

\subsection{Experimental environment}
All experiments were conducted on a server equipped with two AMD EPYC 7543 32-Core processors (64 cores in total) and four NVIDIA RTX A6000 GPUs (48GB memory each), running Ubuntu 24.04 LTS. Synthetic experiments (Sections~\ref{sec:exp_ltf},~\ref{sec:exp_rf}) and IMDB experiment (Section~\ref{sec:IMDB}) were run on CPU, while the LLM fine-tuning experiments (Section~\ref{sec:llm}) utilized the GPUs.

\section{Additional analysis for LTF setting}
\label{sec:ltf_regression}

\begin{figure}[t]
  \centering
  \includegraphics[width=0.7\textwidth]{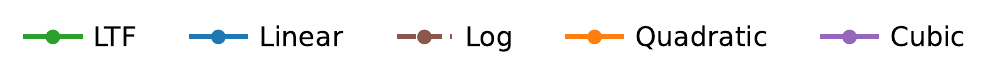}\\[2pt]
  \begin{subfigure}{0.49\textwidth}
    \centering
    \includegraphics[width=\linewidth]{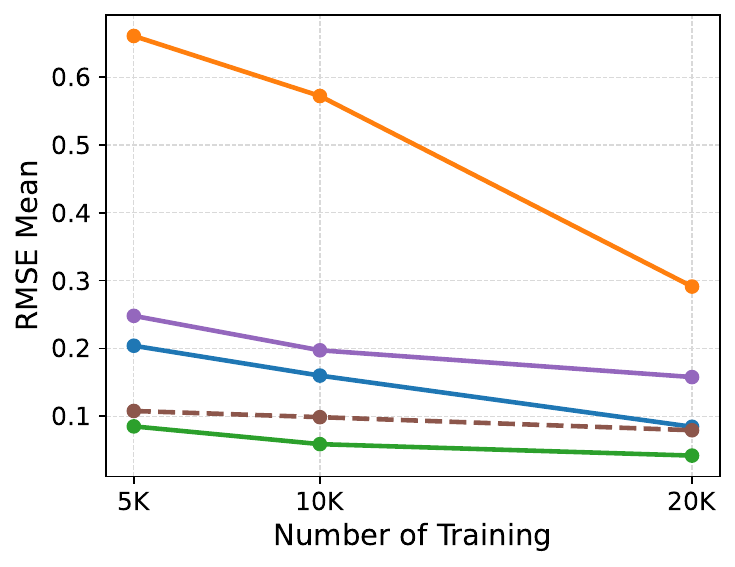}
    \caption{IMDB}
  \end{subfigure}
  \hfill
  \begin{subfigure}{0.49\textwidth}
    \centering
    \includegraphics[width=\linewidth]{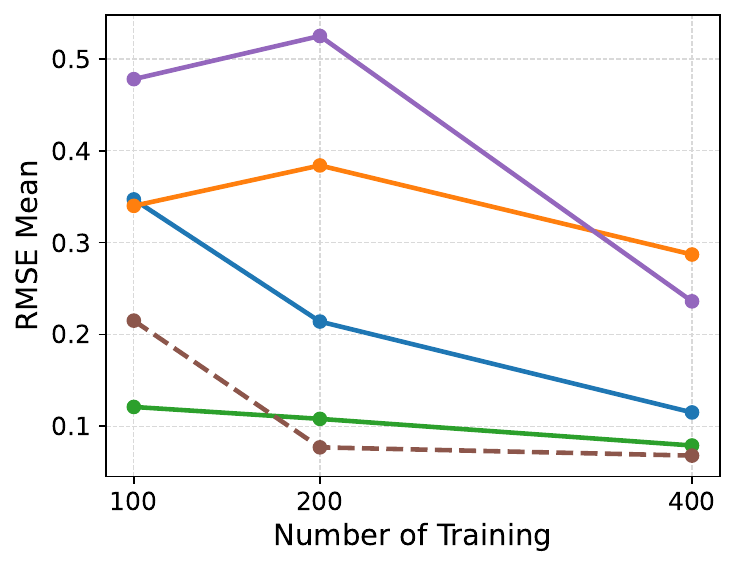}
    \caption{XSum-LLM}
  \end{subfigure}
  \caption{Extrapolation accuracy of parametric families on real reward curves. Each family is fit to an initial prefix of each arm's training curve and evaluated by its RMSE against the full mean reward curve, averaged over arms. Most evaluated points lie beyond the prefix, so the RMSE reflects extrapolation of the remaining reward. Across both tasks, the two-parameter LTF family is consistently among the most accurate and clearly outperforms the higher-degree polynomial fits.}
  \label{fig:ltf_regression}
\end{figure}

We test whether the Linear-Then-Flat (LTF) form predicts real learning curves well from an early observation window. We use the mean reward curves $\mu_i(n)$ from the IMDB model-selection experiment (Figure~\ref{fig:imdb_reward}) and the XSum-LLM fine-tuning experiment (Figure~\ref{fig:llm_reward}). For each arm, we fit each candidate family by least squares to a prefix of the first $B$ pulls and evaluate its RMSE against the full curve, using $B \in \{5K, 10K, 20K\}$ for IMDB and $B \in \{100, 200, 400\}$ for XSum-LLM. Since most evaluated points are unseen during fitting, this RMSE measures extrapolation of the remaining reward rather than fit to the prefix.

We compare five families. The two-parameter LTF model matches Definition~\ref{def:ltf},
\begin{equation*}
\widehat{\mu}_{\mathrm{LTF}}(n) = \min\left\{c_0,\, c_1 \cdot n \right\},
\end{equation*}
with growth rate $c_1$ and maximum attainable reward $c_0$. The other models are parameterized as follows:
\begin{align*}
\text{Linear:}\quad &\widehat{\mu}(n) = c_0 + c_1 z, \\
\text{Quadratic:}\quad &\widehat{\mu}(n) = c_0 + c_1 z + c_2 z^2, \\
\text{Cubic:}\quad &\widehat{\mu}(n) = c_0 + c_1 z + c_2 z^2 + c_3 z^3, \\
\text{Log:}\quad &\widehat{\mu}(n) = c_0 \log(1 + c_1 n / B), \quad B = \text{number of observed data points},
\end{align*}
where $z = 2n/B - 1$ is a time coordinate normalized to the fitting window, using $2$, $3$, $4$, and $2$ parameters, respectively. All predictions are clipped to $[0,1]$ before evaluation.

As shown in Figure~\ref{fig:ltf_regression}, the two-parameter LTF family is consistently among the most accurate on both tasks and clearly outperforms the higher-degree polynomial fits. The polynomial families fit the prefix but extrapolate poorly, since unconstrained polynomial growth does not encode saturation; the linear family cannot represent a transition to a plateau. LTF instead directly captures the two phases seen during training, improvement followed by a plateau once learning has effectively converged, which motivates it as a controlled test bed for reasoning about finite horizon.

\section{Proof of Theorem~\ref{thm:dominance_ltf}}
\label{sec:proof_dominance_ltf}
In this section, we provide the proof for Theorem~\ref{thm:dominance_ltf}.
Since we consider deterministic reward setting (e.g., $\sigma$=0), CURE-UCB and R-ed-UCB \citep{metelli2022stochastic}, we do not need to consider large sliding-window and can just set sliding window parameter $\epsilon=0$ and initialize with selecting all arms once.
For the reader's convenience, we recall index of both algorithms.
Firstly, we recall index of R-ed-UCB in deterministic setting, which is given in Section 4.2 of \citep{metelli2022stochastic}
\begin{align}
    B_{i}^{red}(t) := \mu_{i}(N_{i,t-1}) + (t-N_{i,t-1})\gamma_i(N_{i,t-1}-1) \;. \label{eq:det_red}
\end{align}
The index of CURE-UCB in deterministic setting is given as follows:
\begin{align}
    B_{i}^{cure}(t) := \mu_{i}(N_{i,t-1}) + \frac{(T-t+2)}{2}\gamma_i(N_{i,t-1}-1) \label{eq:det_cure} \;.
\end{align}
In LTF setting, the expected reward of every arm $i\in[K]$ is given as follows:
\begin{align}
    \mu_i(n) = \min(b_i,a_in)\;.
\end{align}
For every arm, we define $c_i := \frac{b_i}{a_i}$, which means the time steps required to make arm $i$ saturated.
For convenience, we fix some $T$ and assume that the arms are ordered in descending order with respect to $a_i$.
That is, $a_1\geq a_2\geq \cdots \geq a_K$.
For convenience, we assume throughout this proof that the initialization 
step has already been completed and that the slope $a_i$ of each arm has 
been estimated exactly.
Let $I_{cure}(T)$ and $I_{red}(T)$ denote the set of arms that are played at least once until time $T$ after initialization by CURE-UCB and by R-ed-UCB respectively.
Then, we assert the following claim:
\begin{claim} \label{clm:1}
    Both $I_{cure}(T)$ and $I_{red}(T)$ grows monotonically with $T$, and its element appear in ascending order $\{1,2,\cdots,K\}$.
\end{claim}
This claim trivially holds because we assume that all arms increase linearly starting from zero. 
Consequently, both algorithms are compelled to select the arms in the order of decreasing order with respect to $a_i$.

Then, we now characterize the behavior of the two algorithms.
Firstly, we consider CURE-UCB:
\begin{claim} \label{clm:2}
    Let $i^{*}_{cure}$ be defined as follows:
    \begin{align}
        i^*_{cure} :=\argmax_{i\in I_{cure}(T)} b_i \;.
    \end{align}
Then, CURE-UCB selects each arm $i\in I_{cure}(T)\setminus i^*_{cure}$ exactly $(c_i+1)$ times, and select arm $i^*_{cure}$ for the remaining $T-(K-|I_{cure}(T)|)-\sum_{i\in I_{cure}\setminus i^*_{cure}}(c_i+1)$ times.
\end{claim}
We begin by an observation that if CURE-UCB selects some arm $i$ for the first time at time $t_i$ after the initialization,  then it must continue to select arm $i$ for $c_i $ consecutive rounds.
The fact that CURE-UCB selects arm $i$ at time $t_i$ 
implies that the index of arm $i$, $B^{cure}_i(t_i)$, is the largest among all arms.
Then, at time $t_i+1$, we observe that $B_i^{cure}(t_i+1) = B^{cure}_i(t_i)+a_i/2$, 
while for any other arm $j \in [K]$, we have $B^{cure}_j(t_i+1) \leq B^{cure}_j(t_i)$. 
Therefore, CURE-UCB also selects arm $i$ at time $t_i+1$. 
This continues until arm $i$ becomes flat, namely for $c_i$ additional selections after initialization.

Also, we observe that if some arm $i$, after being selected $c_i+1$ times, is selected again at some time $t'_i$, then arm $i$ will continue to be selected for the remaining time.
We consider two cases for another arm $j$: 
either by time $t'_i$ it has already been selected $c_j+1$ times, or it has not.
If arm $j$ has already been selected $c_j+1$ times, 
then the fact that arm $i$ is selected at time $t'_i$ implies that $b_i > b_j$. 
It is then immediate that, for the remaining time, 
the index of arm $i$ remains greater than that of arm $j$.
If arm $j$ has been selected fewer than $c_j+1$ times, 
then by the previous observation, arm $j$ has not been selected after initialization yet. 
At time $t'_i+1$, comparing the indices yields
\begin{align}
B^{cure}_i(t'_i+1) = B^{cure}_i(t'_i) 
\end{align}
for arm $i$, and
\begin{align}
B^{cure}_j(t'_i+1) = B^{cure}_j(t'_i) -\frac{a_j}{2}
\end{align}
for arm $j$. 
Then, since $B^{cure}_j(t'_i+1) \leq B^{cure}_j(t'_i) \leq B^{cure}_i(t'_i) = B^{cure}_i(t'_i+1)$, which means that arm $i$ will be selected at time $t'_i+1$.
The same reasoning can be applied repeatedly for every round from $t = t'_i+1$ up to $T$.

Based on the two preceding observations, we prove the claim.
First, if every arm in $I_{cure}(T)$ has been selected at least $c_i+1$ times, 
then by preceding observations, our claim is immediate.
Now, suppose that there exists an arm in $I_{cure}(T)$ that has been selected at most $c_i$ times. 
By the first observation, such an arm must be unique.
Let us denote this unique arm by $i'$.
Let $t'$ denote the time at which arm $i'$ is firstly                                                  selected. 
Then, by the first observation, we have $(T - t'+1) < c_{i'}$.
Define $b'$ as the maximum value of $b_i$ among the arms in $ i\in I_{cure}(T)\setminus i'$, with $b'=0$ when this set is empty.
Then, at time $t'$, we have
\begin{align}
a_{i'}+ \frac{a_{i'}}{2}(T - t'+2) > b'.
\end{align}
Rearranging this inequality gives
\begin{align}
b' < a_{i'}\frac{(T - t'+4)}{2} \leq a_{i'}(T-t'+2)\leq a_{i'}c_{i'} = b_{i'}
\end{align}
Hence, $i'$ corresponds to $i^*_{cure}$, and the proof of the claim is complete.

Now, we make a claim for R-ed-UCB.
\begin{claim} \label{clm:3}
Let $i'_{red}$ be defined as follows:
\begin{align}
    i'_{red} := \argmin_{i \in I_{red}(T)} a_i. 
\end{align}
Then, R-ed-UCB selects each arm  $i \in I_{red}(T) \setminus \{ i'_{red} \}$
at least $(c_i+1)$ times.
\end{claim}
Similarly, we begin by an observation that if R-ed-UCB selects some arm $i$ for the first time at time $t_i$,  then it must continue to select arm $i$ for the next $c_i$ rounds.
By the design of R-ed-UCB, at time $t_i$, among the arms that are not selected, 
arm $i$ is the one with the largest slope. Moreover, all arms with larger slopes 
than arm $i$ must already have been pulled at least $c_i+1$ times. Otherwise, R-ed-UCB would select another arm instead of arm $i$. 
Therefore, it is clear that arm $i$ will be selected from time $t_i$ to $\min\left\{T, t_i + c_i - 1\right\}$.
However, unlike CURE-UCB, note that R-ed-UCB cannot guarantee that these arms are pulled exactly $(c_i+1)$ times.

Then, we make following claim for relationship between $I_{cure}(T)$ and $I_{red}(T)$.
\begin{claim} \label{clm:4}
    For any fixed $T$,
    \begin{align}
        I_{cure}(T) \subseteq I_{red}(T) \;.
    \end{align}
\end{claim}
Let $T_{i}^{red}$ be the first total time when $i$ is selected after initialization by R-ed-UCB.
Then, we can write $T_{i}^{red}$ as follows:
\begin{align*}
    T_{1}^{red} &= K+1 \\
    T_{2}^{red} &\leq \max\left\{T_1^{red}+c_1,\frac{b_1}{a_2}+1\right\} \\ \\
    T_{3}^{red} &\leq \max\left\{T_2^{red}+c_2,\frac{\max(b_1,b_2)}{a_3}+1\right\} \\
    \cdots \\
    T_{k}^{red} &\leq \max\left\{T_{k-1}^{red}+c_{k-1},\frac{\max(b_1,b_2,\cdots,b_{k-1})}{a_k}+1\right\} 
    \;.
\end{align*}
Similarly, let $T_{i}^{cure}$ be the first time when $i$ is selected after initialization by CURE-UCB.
Then, we can write $T_{i}^{cure}$ as follows:
\begin{align*}
     T_{1}^{cure} &= K+1 \\
     T_{2}^{cure} &= K+1+c_1 \\
     T_{3}^{cure} &= K+1+c_1 + c_2  \\
     \cdots \\
     T_{k}^{cure} &= K+1+\sum_{i=1}^{k-1}c_i \;.
\end{align*}
Then, it suffices to show that $T_{i}^{red} \leq T$ for all arms $i\in I_{cure}(T)$.
To show it, we utilize induction for $D_i :=T_{i}^{red}-T_{i}^{cure}\leq \max\left\{\frac{\max(b_1,b_2,\dots,b_{i-1})}{a_i}-K,0\right\}$.
For base case ($i=1$), $D_i=0$.
Then, from previous results, we have:
\begin{align}
    D_k &\leq \max\left\{T_{k-1}^{red}+c_{k-1}-T_{k}^{cure},\frac{\max(b_1,b_2,\cdots,b_{k-1})}{a_k}+1-T_{k}^{cure}\right\} \\
    &= \max\left\{T_{k-1}^{red}+c_{k-1}-T_{k-1}^{cure}-c_{k-1},\frac{\max(b_1,b_2,\cdots,b_{k-1})}{a_k}+1-1-K-\sum_{i=1}^{k-1}c_i\right\} \\
    &= \max\left\{D_{k-1},\frac{\max(b_1,b_2,\cdots,b_{k-1})}{a_k}-K-\sum_{i=1}^{k-1}c_i\right\} \\
    &\leq \max\left\{D_{k-1},\frac{\max(b_1,b_2,\dots,b_{k-1})}{a_k}-K\right\} \\
    &\leq \max\left\{0,\frac{\max(b_1,b_2,\dots,b_{k-2})}{a_{k-1}}-K,\frac{\max(b_1,b_2,\dots,b_{k-1})}{a_k}-K\right\} \\
    &=\max\left\{\frac{\max(b_1,b_2,\dots,b_{k-1})}{a_k}-K,0\right\} \;. \label{eq:ltf1}
\end{align}
At time $T_{k}^{cure}$, CURE-UCB selects arm $k$ and it implies that:
\begin{align}
    \max(b_1,b_2,\dots,b_{k-1}) &\leq a_k + \frac{T-T_{k}^{cure}+2}{2}a_k \\
    \frac{\max(b_1,b_2,\dots,b_{k-1})}{a_k} &\leq \frac{T-T_{k}^{cure}+4}{2} \label{eq:ltf2}
\end{align}
Combining every result, we have:
\begin{align}
    D_k \leq \max\left\{\frac{T-T_{k}^{cure}+4}{2}-K,0\right\} \leq T-T_{k}^{cure} \;.
\end{align}
Therefore, $T_k^{red} \leq T$, which completes proof of Claim 4.
Now, we proceed the proof of main theorem.
Fix any horizon $T$ and define $m:=\max I^{cure}(T)$, which is last explored by CURE-UCB and $S^{cure}(T) := I^{cure}(T) \setminus\left\{m\right\}$.
Then, by Claim 2, for every arm $i\in S^{cure}(T)$, it is guaranteed that $i$ is selected by CURE-UCB at least $c_i$ times.
Also, by Claim 4, every arm $i\in I^{cure}(T)$ are selected by both R-ed-UCB and except for $m$, it is guaranteed that every arm $i\in S^{cure}(T)$ is selected at least $c_i$ times by Claim 3.
Therefore, we can cancel common selections: for every arm $i\in S^{cure}(T)$, both algorithm select it for $c_i$ times.  
Then, we can consider remaining time $R=T-K-\sum_{i=1}^{m-1}c_i$.
We consider two cases.\\

\textit{Case 1: R-ed-UCB selects arm $m$ at most $c_m$ times}.\\
Let $b':=\max_{i\leq m-1}b_i$ and $D:=T_{m}^{red}-T_{m}^{cure}$, $N^{\mathrm{red}}_{m,T}$ be the number of selecting arm $m$ by R-ed-UCB until $T$. 
In this case, R-ed-UCB selects arm $m$ for $R-D$ times and previous plateaus for $D$ times. Then, the cumulative reward of R-ed-UCB is upper bounded by $\sum_{t=2}^{R-D+1}\mu_m(t)+Db'$.
In contrast to this, CURE-UCB selects arm $m$ until its first zero increment or the horizon. If the block finishes, each subsequent
selection yields $\max\{b',b_m\}\geq b_m$.
Thus, its cumulative reward is at least
$\sum_{t=2}^{R+1}\mu_m(t)$.
If $D=0$, the preceding upper and lower bounds already
give the desired comparison.
Then, the cumulative reward of CURE-UCB is lower bound by $\sum_{t=2}^{R+1}\mu_m(t)$.
Then, \eqref{eq:ltf1} and \eqref{eq:ltf2} gives:
\begin{align}
    D\leq \frac{b'}{a_m}-K, \frac{b'}{a_m} \leq \frac{R+3}{2}\;.
\end{align}
Suppose henceforth that $D>0$.
Consequently, we have:
\begin{align}
    N^{\mathrm{red}}_{m,T} = 1+R-D &\geq 1+R-\frac{b'}{a_m}+K \\
    &\geq \frac{b'}{a_m}+K-2 \\
    &\geq \frac{b'}{a_m}\;.
\end{align}
Since $N^{\mathrm{red}}_{m,T} \leq c_m$, we have $\mu_m(R-D+1)= a_m(R-D+1)\geq b'$.
Then, by monotonicity,
\begin{align}
    \sum_{t=2}^{R+1}\mu_m(t) -  \sum_{t=2}^{R-D+1}\mu_m(t)-Db' &= \sum_{t=R-D+2}^{R+1}\mu_m(t)-Db' \\
    &\geq Db'-Db' =0
\end{align}

Therefore, CURE-UCB has cumulative reward no smaller than R-ed-UCB.

\textit{Case 2: R-ed-UCB selects arm $m$ at least $c_m+1$ times}.\\
In this case, we can also cancel arm $m$ for $c_m$ times.
Let $b^{\star}:=\max_{i\leq m}b_i$.
For the remaining time $L=T-K-\sum_{i=1}^{m}c_i$, CURE-UCB selects arm with $b^\star $ for $L$ times and R-ed-UCB may select arm $m+1$.
However, the fact that CURE-UCB does not select arm $m+1$ implies that
\begin{align}
    b^\star \geq a_{m+1}\frac{L+3}{2}\;. 
\end{align}
It implies that for remaining time, the optimal policy is selecting arm with $b^\star$ continuously and therefore CURE-UCB is optimal.

Therefore, for both cases, it is guaranteed that
\begin{align}
    Reg_\mu\!\left(\pi^{\mathrm{cure}}_{0}, T\right) \leq Reg_\mu\!\left(\pi^{\mathrm{red}}_{0}, T\right)  \;.
\end{align}


\section{Proof of Proposition~\ref{prop:linear_gap}}
\label{sec:proof_linear_gap}
\begin{proof}
We first prove the claim for $K = 2$ and then extend it to arbitrary $K$.

\paragraph{Construction ($K=2$).}
Let arm $1$ be a constant early peaker and arm $2$ a late bloomer:
\begin{equation}
    \mu_1(n) = \tfrac{1}{2}, \qquad
    \mu_2(n) = \min\{1,\, a_2 n\}, \quad
    a_2 = \frac{1}{\lfloor \alpha T \rfloor},\  \ \alpha \in (0,1/2),
\end{equation}
so that arm $2$ grows linearly at rate $a_2$ and saturates at pull
$c_2 := \lfloor \alpha T \rfloor$.

\paragraph{Optimal policy.}
By Proposition~\ref{prop:optimal_structure} the optimal policy plays a single arm for every time steps.
The two candidates have cumulative rewards as follows:
\begin{equation}
    \sum_{n=1}^{T}\mu_1(n) = \frac{T}{2},
    \qquad
    \sum_{n=1}^{T}\mu_2(n)
    = a_2\frac{c_2(c_2+1)}{2} + (T - c_2)
    = T - \frac{c_2 - 1}{2}.
\end{equation}
Since $c_2 = \lfloor\alpha T\rfloor \le T$, we have
$T - \tfrac{c_2-1}{2} > \tfrac{T}{2}$, so arm $2$ is the unique optimal arm and
$\mathrm{OPT} = \sum_{n=1}^{T}\mu_2(n)$.

\paragraph{R-ed-UCB.}
In the deterministic setting its index is
$B_i^{\mathrm{red}}(t) = \mu_i(N_{i,t-1}) + (t - N_{i,t-1})\,\gamma_i(N_{i,t-1}-1)$.
Arm $1$ is constant, so $B_1^{\mathrm{red}}(t) = \tfrac{1}{2}$. While arm $2$ is
in its linear phase ($N_{2,t-1} < c_2$),
\begin{equation}
    B_2^{\mathrm{red}}(t)
    = \mu_2(N_{2,t-1}) + (t - N_{2,t-1})\,a_2
    = a_2 N_{2,t-1} + (t - N_{2,t-1})\,a_2
    = a_2 t,
\end{equation}
which depends only on $t$. Hence R-ed-UCB prefers arm $1$ while $a_2 t \le
\tfrac{1}{2}$ and switches to arm $2$ permanently once $t > \tfrac{1}{2a_2} =
\tfrac{c_2}{2}$. It thus pulls the suboptimal arm $1$ exactly
$m := \big\lfloor \tfrac{c_2}{2}\big\rfloor$ times, then plays arm $2$ for the
remaining $T - m$ rounds. Its cumulative reward is
\begin{equation}
    V^{\mathrm{red}} = \frac{m}{2} + \sum_{n=1}^{T-m}\mu_2(n),
\end{equation}
so, using $\mathrm{OPT} = \sum_{n=1}^{T}\mu_2(n)$,
\begin{equation}
    \mathrm{Reg}_{\mu}(\pi^{\mathrm{red}}_{\det}, T)
    = \sum_{n=T-m+1}^{T}\mu_2(n) - \frac{m}{2}
    = m - \frac{m}{2}
    = \frac{m}{2}
    = \left\lfloor \frac{\alpha T}{4}\right\rfloor \geq \frac{\alpha T}{4}-1,
\end{equation}
where the middle step uses that the top $m$ pulls of arm $2$ are all at the
saturated level $1$ (since $T - m > c_2$ for $\alpha < 1/2$). Thus
$\mathrm{Reg}_{\mu}(\pi^{\mathrm{red}}_{\det}, T) \geq \frac{\alpha T}{4}-1$.

\paragraph{CURE-UCB.}
Its index carries the horizon factor,
$B_i^{\mathrm{cure}}(t) = \mu_i(N_{i,t-1}) + \tfrac{T-t+2}{2}\,\gamma_i(N_{i,t-1}-1)$.
Right after initialization, the future-gain term of arm $2$ is
$\tfrac{T-t+2}{2}a_2 = \tfrac{T-t+2}{2c_2}=\tfrac{T-t+2}{2\lfloor\alpha T\rfloor } \geq \tfrac{T-t+2}{T}$,
which exceeds $\tfrac{1}{2}$.
Therefore, $B_2^{\mathrm{cure}} > B_1^{\mathrm{cure}}$
from the first index-based round. Then, as shown in the proof of Theorem~\ref{thm:dominance_ltf}, CURE-UCB commits to arm $2$ until it saturates.
After then, it is clear that arm $2$ has larger value than arm 1, which leading to committing arm $2$ for remaining time steps.
Therefore, the CURE-UCB commits arm $2$ for every time step except for initialization steps and the cumulative regret is:
\begin{equation}
    \mathrm{Reg}_{\mu}(\pi^{\mathrm{cure}}_{\det}, T) = \tfrac{1}{2}.
\end{equation}

Combining the two,
\begin{equation}
    \mathrm{Reg}_{\mu}(\pi^{\mathrm{red}}_{\det}, T)
    - \mathrm{Reg}_{\mu}(\pi^{\mathrm{cure}}_{\det}, T)
    \geq \frac{\alpha}{4}T - \frac{3}{2}
    = \Omega(T),
\end{equation}
which establishes the $K=2$ case with any $\alpha \in (0,1/2)$.

\paragraph{Extension to $K \ge 2$.}
Augment the two-armed instance with $K - 2$ dummy arms that saturate after a
single pull at a constant level below $\tfrac{1}{2}$, e.g.\ $\mu_i(n) =
\tfrac{1}{4}$ for $i \ge 3$. These arms have zero growth ($\gamma_i = 0$)
and reward strictly below both arm $1$ and the saturated arm $2$, so after the
one initialization pull each, their index equals $\tfrac{1}{4}$ under both
algorithms and is never maximal again. Consequently neither algorithm selects a
dummy arm after initialization, the selection dynamics of arms $1$ and $2$ are
unchanged, and each dummy contributes only an $O(1)$ additive term to both
regrets, which cancels in the difference. Therefore
$\mathrm{Reg}_{\mu}(\pi^{\mathrm{red}}_{\det}, T)
- \mathrm{Reg}_{\mu}(\pi^{\mathrm{cure}}_{\det}, T) = \Omega(T)$
holds for every $K \ge 2$.
\end{proof}

\section{Proof of Theorem~\ref{thm:upper_bound}}
\label{sec:proof_upper_bound}
{
\allowdisplaybreaks
\begin{proof}
 Let $\tau:=K\lceil2/\epsilon\rceil$ be the number of initialization time steps.
 For well-estimated event, we define $\bar{\mu}_i(t)$ and $\widetilde{\mu}_i(t)$ as follows:
\begin{align}                                                               
    &\bar{\mu}_i(t) := \frac{1}{h_i}\sum_{l=N_{i,t-1}-h_i+1}^{N_{i,t-1}}\left(X_i(l)+\frac{(T-t+2)}{2}\frac{X_i(l)-X_i(l-h_i)}{h_i}\right)\;, \\
    &\widetilde{\mu}_i(t) := \frac{1}{h_i}\sum_{l=N_{i,t-1}-h_i+1}^{N_{i,t-1}}\left(\mu_i(l)+\frac{(T-t+2)}{2}\frac{\mu_i(l)-\mu_i(l-h_i)}{h_i}\right)\;, \\
    &\beta_i(t) := \sigma \sqrt{\frac{2[3(T-t+2)^2+8h_i^2]\log t^3}{4h_i^3}} \;,\\
    &B_i(t) := \bar{\mu}_i(t) + \beta_i(t) \;.
\end{align}
We define well-estimated event $\mathcal{E}_{t}$ as follows:
\begin{align}
    &\mathcal{E}_{i,t}:=\left\{|\widetilde{\mu}_i(t)-\bar{\mu}_i(t)|\leq \beta_i(t) \right\}\; ,\\
    &\mathcal{E}_t:=\bigcap_{i\in [K]}\mathcal{E}_{i,t}\;.
\end{align}

Then, we rewrite the regret as follows:
\vspace{-0.1cm}
\begin{align} 
    Reg_{\vec{\mu}}(\pi,T) &= \sum_{t=1}^{T}\EXP\left[\mu_{i^*}(t) - \mu_{i_{t}}(N_{i_t,t})\right] \\
    &\leq \tau+\sum_{t=\tau+1}^{T}\EXP\left[(\mu_{i^*}(t) - \mu_{i_t}(N_{i_t,t}))\ind\left\{\gE_t\right\}\right]  \notag \\ &+\sum_{t=\tau+1}^{T}\EXP\left[(\mu_{i^*}(t) - \mu_{i_t}(N_{i_t,t}))\ind\left\{\gE_t^c\right\}\right] \\
    &\leq \tau+\underbrace{\sum_{t=\tau+1}^{T}\EXP\left[(\mu_{i^*}(t) - \mu_{i_t}(N_{i_t,t}))\ind\left\{\gE_t\right\}\right]}_{(A)}+ \underbrace{\sum_{t=\tau+1}^{T}\EXP\left[\ind\left\{\gE_t^c\right\}\right]}_{(B)}\;.
\end{align}

Then, we bound (B) term.
\begin{align}
    (B) &= \sum_{t=\tau+1}^{T}\Pr\left(\gE_t^c\right) \\
    &=  \sum_{t=\tau+1}^{T}\Pr\left(\bigcup_{i\in[K]}\gE_{i,t}^c\right)  \label{eq:1} \\
    &\leq \sum_{t=\tau+1}^{T}\sum_{i\in[K]}\Pr(\gE^c_{i,t})  \label{eq:2} \;,
\end{align}
where \eqref{eq:1} holds due to De Morgan's Law and \eqref{eq:2} follows from the application of the union bound. 
From Lemma~\ref{lem:concen}, we have:
\begin{align}
    \sum_{t=\tau+1}^{T}\sum_{i\in[K]}\Pr(\gE^c_{i,t})  \leq \sum_{t=\tau+1}^{T}\sum_{i\in[K]} 2t^{-2} \leq \frac{2K}{\tau} 
\;.    
\end{align}

Now, we bound (A) term.
Under $\gE$, we can decompose (A) term as follows:
\begin{align}
    \mu_{i^*}(t) - \mu_{i_t}(N_{i_t,t}) &\leq \mu_{i^*}(t) - \mu_{i_t}(N_{i_t,t}) + B_{i_t}(t) - B_{i^*}(t) \\
    &= \mu_{i^*}(t) - \mu_{i_t}(N_{i_t,t}) + (\bar{\mu}_{i_t}(t)+\beta_{i_t}(t)) - (\bar{\mu}_{i^*}(t)+\beta_{i^*}(t))\\
    &\leq \mu_{i^*}(t) - \mu_{i_t}(N_{i_t,t}) + (\widetilde{\mu}_{i_t}(t)+2\beta_{i_t}(t)) - \widetilde{\mu}_{i^*}(t) \label{eq:3}\\
    &= \underbrace{\mu_{i^*}(t) - \widetilde{\mu}_{i^*}(t)}_{(A1)} + \underbrace{\widetilde{\mu}_{i_t}(t) - \mu_{i_t}(N_{i_t,t})}_{(A2)} + \underbrace{2\beta_{i_t}(t)}_{(A3)} \label{eq:upper1}\;,
\end{align}
where \eqref{eq:3} holds by the definition of $\gE$.
Write $[x]_{+} :=\max\left\{x,0\right\}$.
Then \eqref{eq:upper1} implies that:
\begin{align}
    (A) \leq \EXP\left[\sum_{t=\tau+1}^{T}(\min\left\{1,[(A1)]_+\right\})+\min\left\{1,[(A2)]_+\right\}+\min\left\{1,[(A3)]_+\right\}\right]\;.
\end{align}
Then, we bound (A1) term.
We can rewrite (A1) term as follows:

\begin{align}
    (A1) &= \mu_{i^*}(t) - \frac{1}{h_i^*}\sum_{l=N_{i^*,t-1}-h_{i^*}+1}^{N_{i^*,t-1}}\left(\mu_i(l) +\frac{(T-t+2)}{2}\frac{\mu_{i^*}(l)-\mu_{i^*}(l-h_{i^*})}{h_{i^*}}\right) \\
    &= \frac{1}{h_{i^*}}\sum_{l=N_{i^*,t-1}-h_{i^*}+1}^{N_{i^*,t-1}} \left(\mu_{i^*}(t) - \mu_{i^*}(l) - \frac{(T-t+2)}{2}\frac{\mu_{i^*}(l)-\mu_{i^*}(l-h_{i^*})}{h_{i^*}}\right) \\
    &= \frac{1}{h_{i^*}}\sum_{l=N_{i^*,t-1}-h_{i^*}+1}^{N_{i^*,t-1}} \left(\mu_{i^*}(t) - \mu_{i^*}(l) - \frac{(T-t+2)}{2h_{i^*}}\sum_{k=l-h_{i^*}}^{l-1}\gamma_{i^*}(k)\right) \label{eq:4} \\
    &\leq \frac{1}{h_{i^*}}\sum_{l=N_{i^*,t-1}-h_{i^*}+1}^{N_{i^*,t-1}} \left(\mu_{i^*}(t) - \mu_{i^*}(l) - \frac{(T-t+2)}{2}\gamma_{i^*}(l-1)\right) \label{eq:5} \\
    &\leq \frac{1}{h_{i^*}}\sum_{l=N_{i^*,t-1}-h_{i^*}+1}^{N_{i^*,t-1}} \left[t-l-\left\lfloor\frac{T-t+2}{2}\right\rfloor\right]_+\gamma_{i^*}\left(\left\lfloor\frac{T-t+2}{2}\right\rfloor+1\right) \label{eq:6} \\
    &\leq \frac{1}{h_{i^*}}\sum_{l=N_{i^*,t-1}-h_{i^*}+1}^{N_{i^*,t-1}} T\gamma_{i^*}\left(\left\lfloor\frac{T-t+2}{2}\right\rfloor+1\right)  \\
    &= T\gamma_{i^*}\left(\left\lfloor\frac{T-t+2}{2}\right\rfloor+1\right)  
    \;,
\end{align} 
where \eqref{eq:4} holds by the definition of $\gamma_i$ and \eqref{eq:5} holds by concavity assumption, (i.e., $\gamma_i(n)$ is decreasing over $n$) and \eqref{eq:6} holds by Lemma~\ref{lem:future_bound}.

With summation, we have:
\begin{align}
    \sum_{t=\tau+1}^{T} (A1)   &\leq \sum_{t=\tau+1}^{T}T\gamma_{i^*}\left(\left\lfloor\frac{T-t+2}{2}\right\rfloor+1\right) \\
    &\leq \sum_{t=\tau+1}^{T}T\gamma_{\max}\left(\left\lfloor\frac{T-t+2}{2}\right\rfloor+1\right)  \\
    &\leq \sum_{t=2}^{\lceil(T-\tau)/2\rceil+1} 2T\gamma_{\max}(t) \label{eq:7}\\
    &\leq 2T\sum_{t=1}^{T} \gamma_{\max}(t)\;,
\end{align}
where \eqref{eq:7} holds because each integer increment index occurs at most twice.
Then, we have:
\begin{align}
    \sum_{t=\tau+1}^{T}\min\left\{1,[(A1)]_+\right\} 
    &\leq \sum_{t=\tau+1}^{T}\min\left\{1,T\gamma_{i^*}\left(\left\lfloor\frac{T-t+2}{2}\right\rfloor+1\right)  \right\}\\
    &\leq \sum_{t=\tau+1}^{T}\min\left\{1,T\gamma_{max}\left(\left\lfloor\frac{T-t+2}{2}\right\rfloor+1\right)  \right\}\\
    &\leq T^q\sum_{t=\tau+1}^{T}\gamma_{max}\left(\left\lfloor\frac{T-t+2}{2}\right\rfloor+1\right)^q\label{eq:8}\\
    &\leq \frac{KT^{q}}{1-2\epsilon}\Upsilon_{\vec{\mu}}\left(\left\lceil(1-2\epsilon)\frac{T}{K}\right\rceil,q\right) \label{eq:9}
    \;,
\end{align}
where \eqref{eq:8} holds since when $\min\left\{1,x\right\}\leq\min\left\{1,x\right\}^q\leq x^q$ for any $x\geq0$ and \eqref{eq:9} holds by Lemma C.2. from \citep{metelli2022stochastic}.

Then, we bound (A2) term.
\begin{align}
    (A2) &=  \frac{1}{h_{i_t}}\sum_{l=N_{i_t,t-1}-h_{i_t}+1}^{N_{i_t,t-1}}\left(\mu_{i_t}(l)+\frac{(T-t+2)}{2}\frac{\mu_{i_t}(l)-\mu_{i_t}(l-h_{i_t})}{h_{i_t}}\right) - \mu_{i_t}(N_{i_t,t}) \\
    &= \frac{1}{h_{i_t}}\sum_{l=N_{i_t,t-1}-h_{i_t}+1}^{N_{i_t,t-1}}\left(\mu_{i_t}(l)+\frac{(T-t+2)}{2}\frac{\mu_{i_t}(l)-\mu_{i_t}(l-h_{i_t})}{h_{i_t}} - \mu_{i_t}(N_{i_t,t})\right) \\
    &\leq \frac{1}{h_{i_t}}\sum_{l=N_{i_t,t-1}-h_{i_t}+1}^{N_{i_t,t-1}} \frac{(T-t+2)}{2}\frac{\mu_{i_t}(l)-\mu_{i_t}(l-h_{i_t})}{h_{i_t}} \label{eq:10} \\
    &= \frac{(T-t+2)}{2h_{i_t}}\sum_{l=N_{i_t,t-1}-h_{i_t}+1}^{N_{i_t,t-1}}\sum_{k=l-h_{i_t}}^{l-1}\frac{1}{h_{i_t}}\gamma_{i_t}(k) \\
    &\leq \frac{(T-t+2)}{2h_{i_t}}\sum_{l=N_{i_t,t-1}-h_{i_t}+1}^{N_{i_t,t-1}}\gamma_{i_t}(l-h_{i_t}) \label{eq:11} \\
    &\leq \frac{(T-t+2)}{2}\gamma_{i_t}(N_{i_t,t-1}-2h_{i_t}+1) \;,
\end{align}
where \eqref{eq:10} holds since $\mu_i(n)$ is non-decreasing and $l \leq N_{i_t,t-1}\leq N_{i_t,t}$ and \eqref{eq:11} holds by concavity assumption.

Then, we have:
\begin{align}
    \sum_{t=\tau+1}^{T} \min\left\{1,[(A2)]_+\right\}&\leq \sum_{t=\tau+1}^{T}\min\left\{1, \frac{(T-t+2)}{2}\gamma_{i_t}(N_{i_t,t-1}-2h_{i_t}+1)\right\} \\
    &\leq \sum_{t=\tau+1}^{T}\min\left\{1,\frac{T}{2}\gamma_{i_t}(N_{i_t,t-1}-2h_{i_t}+1)\right\} \\
    &\leq \sum_{i=1}^{K}\sum_{n=\lceil2/\epsilon\rceil}^{N_{i,T}-1}\min\left\{1,\frac{T}{2}\gamma_{i}(n-2h_i+1)\right\}  \\
    &\leq \frac{KT^q}{2^q(1-2\epsilon)}\Upsilon_{\vec{\mu}}\left(\left\lceil(1-2\epsilon)\frac{T}{K}\right\rceil,q\right) \label{eq:12},
\end{align}
where \eqref{eq:12} holds by Lemma C.2. from \citep{metelli2022stochastic}.

Then, we bound (A3) term.
We also utilize the fact that the regret for one time step cannot precede 1.
\begin{align}
    \sum_{t=\tau+1}^{T} \min\left\{1,[(A3)]_+\right\} 
    &\le \sum_{i=1}^{K} \sum_{j=1}^{N_{i,T}} \min \left\{ 1, 2\sigma \sqrt{\frac{2[3T^2 + 8(\epsilon j)^2] 3 \log T}{4 (\epsilon j)^3}} \right\} \notag \\
    &\le  \sum_{i=1}^{K} \sum_{j= 1}^{N_{i,T}} \min \left\{ 1, \frac{\sigma T \sqrt{\log T}}{\epsilon^{3/2} j^{3/2}} \sqrt{18+48\epsilon^2} \right\} \notag \\
    &\leq K\sum_{j=1}^{\infty} \min \left\{ 1, \frac{C}{j^{3/2}} \right\} \quad \text{where } C = \sigma T \sqrt{\log T} \epsilon^{-3/2} \sqrt{18+48\epsilon^2} \notag \\
    &\leq K\int_{0}^{\infty} \min\left\{1,C x^{-3/2}\right\} dx\\
    &\le K \left(j^* + \int_{j^*}^{\infty} C x^{-3/2} dx \right) \quad \text{with } j^* = C^{2/3}  \notag \\
    &\le 3K C^{2/3} \notag \\
    &=3K \left( \sigma T \sqrt{\log T} \epsilon^{-3/2} \sqrt{18+48\epsilon^2} \right)^{2/3} \notag \\
    &= \frac{3K}{\epsilon} (\sigma T)^{2/3} (\log T)^{1/3} (18+48\epsilon^2)^{1/3} \;.
\end{align}

Combining all results, we have:
\begin{align}
    Reg_{\mu}(\pi_{\epsilon}^{\mathrm{cure}}, T) \leq & 2K\left\lceil\frac{2}{\epsilon}\right\rceil + \frac{K(1+2^{-q})T^{q}}{1-2\epsilon}\Upsilon_{\vec{\mu}}\left(\left\lceil(1-2\epsilon)\frac{T}{K}\right\rceil,q\right) \\
    &+ \frac{3K}{\epsilon} (\sigma T)^{2/3} (\log T)^{1/3} (18+48\epsilon^2)^{1/3} \notag
\end{align}
\end{proof}
}

\section{Known-horizon extension of Thm 4.6 in \texorpdfstring{\citep{fiandri2024rising}}{Fiandri et al. (2024)}}
\label{sec:lower_bound}

We first provide formal statement for lower bound.
\begin{theorem}[Known-horizon extension of the lower bound,  \texorpdfstring{\citep{fiandri2024rising}}{}] \label{thm:lb}
Let $\pi$ denote arbitrary policy (including horizon-aware).
For any $T \ge 4$ and $q \in [0,1]$, there exists a deterministic ($\sigma = 0$) two-armed
instance $\vec{\mu}^{\star}$, with non-decreasing and concave reward functions bounded in
$[0,1]$, such that
\begin{align}
  Reg_{\vec{\mu}^{\star}}(\pi, T) \;\ge\; \frac{T^{q}}{64}\, \Upsilon_{\vec{\mu}^{\star}}(T, q).    
\end{align}
\end{theorem}

\begin{proof}
Let $m := \lceil T/2 \rceil$ and $r := T - m = \lfloor T/2 \rfloor$. We construct two deterministic
instances $\vec{\mu}^{L}$ and $\vec{\mu}^{P}$. The expected reward of arm $2$ in each instance is
given as follows:
\begin{align}
  \mu_2^{L}(n) := \min\!\Bigl\{\tfrac{n}{T},\, 1\Bigr\}, \qquad
  \mu_2^{P}(n) := \min\!\Bigl\{\tfrac{n}{T},\, \tfrac{m}{T}\Bigr\} .
\end{align}
Thus arm $2$ grows linearly up to level $1$ in $\vec{\mu}^{L}$, whereas it saturates after its $m$-th
pull in $\vec{\mu}^{P}$. Let $\bar{\mu}_2^{L}$ and $\bar{\mu}_2^{P}$ denote the horizon-$T$ average
rewards. A direct calculation gives
\begin{align}
  \bar{\mu}_2^{L} = \frac{T+1}{2T}, \qquad
  \bar{\mu}_2^{P} = \frac{1}{T^2}\Bigl[\tfrac{m(m+1)}{2} + m(T-m)\Bigr], \qquad
  \bar{\mu}_2^{L} - \bar{\mu}_2^{P} = \frac{r(r+1)}{2T^2} .
\end{align}
We then define
\begin{align}
  c := \frac{\bar{\mu}_2^{L} + \bar{\mu}_2^{P}}{2}, \qquad
  \Delta := \frac{\bar{\mu}_2^{L} - \bar{\mu}_2^{P}}{2} = \frac{r(r+1)}{4T^2},
\end{align}
and set arm $1$ to be constant in both instances, $\mu_1^{L}(n) = \mu_1^{P}(n) := c$. Since
$0 < \bar{\mu}_2^{P} < c < \bar{\mu}_2^{L} \le 1$, all rewards lie in $[0,1]$, and every reward function
is non-decreasing with non-increasing increments, hence concave (Assumption~\ref{asu:concavity}).

We first bound the regret on each instance. Let $N_{i,T}$ be the number of pulls of arm $i$ up to
time $T$, so that $N_{1,T} + N_{2,T} = T$. For any non-decreasing sequence, the average of its first
$N_{2,T}$ terms is at most its horizon-$T$ average, so pathwise
$\sum_{n=1}^{N_{2,T}}\mu_2(n) \le N_{2,T}\,\bar{\mu}_2$. In $\vec{\mu}^{L}$, the optimal value is at
least that of always playing arm $2$, and thus
\begin{align}
  \mathrm{Reg}_{\vec{\mu}^{L}}(\pi, T)
  &\ge T\bar{\mu}_2^{L} - \mathbb{E}_{L}\Bigl[c\,N_{1,T} + \textstyle\sum_{n=1}^{N_{2,T}}\mu_2^{L}(n)\Bigr] \\
  &\ge T(c+\Delta) - c\,\mathbb{E}_{L}[N_{1,T}] - (c+\Delta)\,\mathbb{E}_{L}[N_{2,T}]
  = \Delta\,\mathbb{E}_{L}[N_{1,T}],
\end{align}
where the last equality uses $N_{1,T} + N_{2,T} = T$. Symmetrically, in $\vec{\mu}^{P}$ the optimal
value is at least that of always playing arm $1$, which yields
\begin{align}
  \mathrm{Reg}_{\vec{\mu}^{P}}(\pi, T) \ge \Delta\,\mathbb{E}_{P}[N_{2,T}] .
\end{align}

Next, we couple the two instances by running $\pi$ with the same internal randomness in both. We
claim that the two histories coincide pathwise through round $m$. Suppose they agree up to time
$t-1 < m$. Then $\pi$ receives the same horizon $T$, the same history, and the same internal
randomness in both instances, so it selects the same arm at time $t$; a pull of arm $1$ yields reward
$c$ in both instances, and the $k$-th pull of arm $2$ satisfies $k \le t \le m$, so
$\mu_2^{L}(k) = \mu_2^{P}(k) = k/T$. By induction the histories agree for all $t \le m$, and in
particular $N_{i,m}$ is identical across the two instances. Note that this argument applies directly
to a horizon-aware policy: the known horizon $T$ is identical in both instances and is therefore one
of the coinciding inputs through round $m$, so knowledge of $T$ cannot help $\pi$ distinguish
$\vec{\mu}^{L}$ from $\vec{\mu}^{P}$ before round $m$. Using $N_{i,T} \ge N_{i,m}$, we obtain
\begin{align}
  \mathrm{Reg}_{\vec{\mu}^{L}}(\pi, T) + \mathrm{Reg}_{\vec{\mu}^{P}}(\pi, T)
  &\ge \Delta\bigl(\mathbb{E}_{L}[N_{1,m}] + \mathbb{E}_{P}[N_{2,m}]\bigr) \\
  &= \Delta\,\mathbb{E}_{L}[N_{1,m} + N_{2,m}] = \Delta m ,
\end{align}
where the equality uses that $N_{i,m}$ coincides across the two instances. Consequently,
\begin{align}
  \max\bigl\{\mathrm{Reg}_{\vec{\mu}^{L}}(\pi, T),\, \mathrm{Reg}_{\vec{\mu}^{P}}(\pi, T)\bigr\}
  \ge \frac{\Delta m}{2} .
\end{align}

Finally, we normalize by the cumulative increment. Arm $1$ is constant, so $\gamma_{\max}$ is
governed by arm $2$. In $\vec{\mu}^{L}$ we have $\gamma_2^{L}(\ell) = 1/T$ for $1 \le \ell \le T-1$,
and in $\vec{\mu}^{P}$ we have $\gamma_2^{P}(\ell) = 1/T$ for $\ell \le m-1$ and $0$ afterward, so
\begin{align}
  \Upsilon_{\vec{\mu}^{L}}(T, q) = (T-1)\,T^{-q}, \qquad
  \Upsilon_{\vec{\mu}^{P}}(T, q) = (m-1)\,T^{-q} \le (T-1)\,T^{-q} .
\end{align}
Let $\vec{\mu}^{\star} \in \{\vec{\mu}^{L}, \vec{\mu}^{P}\}$ attain the larger regret. Then
\begin{align}
  \frac{\mathrm{Reg}_{\vec{\mu}^{\star}}(\pi, T)}{\Upsilon_{\vec{\mu}^{\star}}(T, q)}
  \ge \frac{\Delta m / 2}{(T-1)\,T^{-q}}
  = \frac{m\,r(r+1)}{8\,T^{2}(T-1)}\, T^{q} ,
\end{align}
where the inequality uses $\Upsilon_{\vec{\mu}^{\star}}(T,q) \le (T-1)T^{-q}$. Since $m \ge T/2$,
$r \ge (T-1)/2$, and $r+1 \ge T/2$, we have $m\,r(r+1) \ge (T-1)T^{2}/8$, and therefore
\begin{align}
  \frac{\mathrm{Reg}_{\vec{\mu}^{\star}}(\pi, T)}{\Upsilon_{\vec{\mu}^{\star}}(T, q)}
  \ge \frac{T^{q}}{64} ,
\end{align}
which is the claimed bound.
\end{proof}

\section{Proof of Lemmas}
\label{sec:proof_lemma}
\begin{lemma} \label{lem:concen}
    For every $\epsilon\in (0,1/2)$, arm $i\in[K]$, time $\tau <t \leq T$, where $\tau=K\lceil2/\epsilon\rceil$ sliding window $h_i = \lfloor\epsilon N_{i,t-1}\rfloor$, we have:
    \begin{align}
        \Pr (\left|\widetilde{\mu}_i(t) - \bar{\mu}_i(t) \right| > \beta_i(t)) \leq 2t^{-2}
    \end{align}
\end{lemma}
\begin{proof}
    Since $N_{i,t-1}\geq \lceil2/\epsilon\rceil$ after initialization, 
    \begin{align}
        \Pr (\left|\widetilde{\mu}_i(t)  - \bar{\mu}_i(t) \right| > \beta_i(t)) &= \Pr (\exists n\in\left\{\lceil2/\epsilon\rceil, \cdots, t-1\right\}: \left|\widetilde{\mu}_i(t)  - \bar{\mu}_i(t)\right|>\beta_i(t), N_{i,t-1}=n) \\
        &\leq \sum_{n=\lceil 2/\epsilon\rceil}^{t-1} \Pr (\left|\widetilde{\mu}_i(t)  - \bar{\mu}_i(t) \right| > \beta_i(t), N_{i,t-1}=n) \label{eq:concen_1}
        \;.
    \end{align}
    Fixing $n$, we have:
    \begin{align}
        &h_i(t)(\widetilde{\mu}_i(t)  - \bar{\mu}_i(t)) \\ 
         = &\sum_{l=n-h_i+1}^{n}\left((\mu_i(l)-X_i(l))+\frac{(T-t+2)}{2}\frac{(\mu_i(l)-X_i(l))-(\mu_i(l-h_i)-X_i(l-h_i))}{h_i}\right) \\
        = &\sum_{l=n-h_i+1}^{n}\left(1+\frac{(T-t+2)}{2h_i}\right)(\mu_i(l)-X_i(l)) - \sum_{l=n-h_i+1}^{n}\frac{(T-t+2)}{2h_i}\left(\mu_i(l-h_i)-X_i(l-h_i)\right) \;.
    \end{align}
\end{proof}
Since $X_i(l)$ is only drawn from $D_i(l)$, which only depends on $l$, all random variables $X_i(l)$ and $X_i(l-h_i)$ are independent.
Also, since $\EXP[X_i(l)]=\mu_i(l)$, we can define $Z_i(l) := X_i(l)-\mu_i(l)$ and apply Azuma-Hoeffding inequality \citep{metelli2022stochastic}.
For that, we compute the sum of square constants:
\begin{align}
    &\sum_{l=n-h_i+1}^{n}\left[\left(1+\frac{(T-t+2)}{2h_i}\right)^2 + \left(\frac{(T-t+2)}{2h_i}\right)^2\right] \\
    \leq &\frac{h_i}{4h_i^2}\left[\left(2h_i+(T-t+2)\right)^2+(T-t+2)^2\right]  \\
    \leq &\frac{1}{4h_i} \left[3(T-t+2)^2 + 8h_i^2\right] \;, \label{eq:concen_2}
\end{align}
where \eqref{eq:concen_2} holds since $(a+b)^2 \leq 2a^2+2b^2$.
Thus, we have:
\begin{align}
    &\Pr (\left|\widetilde{\mu}_i(t)  - \bar{\mu}_i(t) \right| > \beta_i(t)) \\
    =&\Pr \left(\left|\sum_{l=n-h_i+1}^{n}\left[\left(1+\frac{(T-t+2)}{2h_i}\right)Z_i(l) - \frac{(T-t+2)}{2h_i}Z_i(l-h_i)\right]\right| >h_i\beta_i(t)\right) \\
    &\leq 2\exp\left(-\frac{(h_i\beta_i(t))^2}{2\sigma^2\left(\frac{3(T-t+2)^2+8h_i^2}{4h_i}\right)}\right) \\
    &\leq 2t^{-3}\;.
\end{align}
Substituting this result into \eqref{eq:concen_1}, we have:
\begin{align}
    \sum_{n=\lceil2/\epsilon\rceil}^{t-1} \Pr (\left|\tilde{\mu}_i(t) - \bar{\mu}_i(t) \right| > \beta_i(t), N_{i,t-1}=n) \leq \sum_{n=1}^{t-1}2t^{-3} \leq 2t^{-2} 
    \;,
\end{align}
which concludes the proof.

\begin{lemma} \label{lem:future_bound}
For every arm $i\in [K]$, time $t\in [T]$ and any number of pulls $2 \leq n \leq t$, we have:
\begin{align}
    \mu_{i}(t)-\mu_i(n)-\frac{T-t+2}{2}\gamma_i(n-1) \leq \max\left\{0, \left(t-n-\left\lfloor\frac{T-t+2}{2}\right\rfloor\right)\gamma_i\left(\left\lfloor\frac{T-t+2}{2}\right\rfloor+1\right)\right\}
\end{align}
\begin{proof}
    Firstly, we consider $t \leq n+\left\lfloor(T-t+2)/2\right\rfloor$. 
    Then:
    \begin{align}
        \mu_i(t)  &\leq \mu_i(n) + \sum_{l=n}^{t-1} \gamma_i(l) \\
        &\leq \mu_i(n) + (t-n)\gamma_i(n-1)  \label{eq:lem1} \\
        &\leq \mu_i(n) + \frac{T-t+2}{2}\gamma_i(n-1)\;,
    \end{align}
    where \eqref{eq:lem1} holds by concavity assumption (i.e., $\gamma_i$ is non-increasing).
    Since $t-n-\left\lfloor(T-t+2)/2\right\rfloor\leq0$, we can easily see that our claim holds.
    Then, we consider the case when $t > n+\left\lfloor(T-t+2)/2\right\rfloor$.
    We decompose $\mu_{i}(t)-\mu_i(n)$ as follows:
    \begin{align}
        \mu_{i}(t)-\mu_i(n) &= \sum_{l=n}^{t-1}\gamma_i(l) \\
        &= \sum_{l=n}^{n+\left\lfloor(T-t+2)/2\right\rfloor-1} \gamma_i(l) + \sum_{l=n+\left\lfloor(T-t+2)/2\right\rfloor)}^{t-1}\gamma_i(l)  \\
        &\leq \left\lfloor\frac{(T-t+2)}{2}\right\rfloor \gamma_i(n-1) + \left(t-n-\left\lfloor\frac{(T-t+2)}{2}\right\rfloor\right)\gamma_i\left(n+\left\lfloor\frac{(T-t+2)}{2}\right\rfloor\right) \label{eq:fu_1} \\
        &\leq \frac{(T-t+2)}{2}\gamma_i(n-1) + \left(t-n-\left\lfloor\frac{(T-t+2)}{2}\right\rfloor\right)\gamma_i\left(\left\lfloor\frac{(T-t+2)}{2}\right\rfloor+1\right)\label{eq:fu_2}\;,
    \end{align}
    where \eqref{eq:fu_1} and \eqref{eq:fu_2} hold by concavity assumption, which completes proof.
\end{proof}
\end{lemma}

\section{Baselines}
\label{sec:baselines}

\subsection{Description}
\begin{itemize}
    \item R-ed-UCB \citep{metelli2022stochastic}
    is a representative algorithm for the RMAB framework. It utilizes a derivative-based estimator to capture the growth rate of the reward function. By estimating this slope, the algorithm constructs an index designed to predict the instantaneous expected reward at the specific future time step.
    \item ET-SWGTS \citep{fiandri2025thompson}
    is a Thompson Sampling-based algorithm adapted for RMAB framework. It addresses the non-stationary nature of reward growth by combining an initial exploration phase with a sliding-window mechanism. The algorithm first forces exploration to gather sufficient preliminary data for all arms. Subsequently, it updates the Gaussian posterior distributions using only the most recent observations within a fixed window, thereby forgetting outdated low rewards to track the current rising trend.
    \item TI-UCB \citep{xia2024llm}
    is tailored for the RMAB framework, specifically focusing on the increasing-then-converging reward trend. The core mechanism of the algorithm revolves around detecting the convergence point, where the critical moment when an arm's growth saturates and performance stabilizes. By focusing on identifying this transition to the asymptotic limit, TI-UCB aims to select the arm with the highest ultimate capacity. 
    \item R-ed-AE \citep{amichay2025rising}
    addresses the RMAB problem by specifically targeting environments where rewards exhibit linear drifts. The algorithm operates on a successive elimination principle, maintaining a set of active candidate arms. By estimating the growth parameters and future potential of each arm, R-ed-AE progressively discards arms that are statistically inferred to be suboptimal, thereby concentrating computational resources on identifying the most promising candidate.
    \item  SW-UCB \citep{garivier2011upper}
     is a modification of the standard UCB policy designed to handle non-stationary environments. It relies on a sliding window of fixed size to estimate the expected reward of each arm. By exclusively utilizing the most recent observations and discarding outdated history, the algorithm effectively forgets past behavior to adapt to dynamic changes in the reward distribution.
    \item SW-KL-UCB \citep{garivier2011kl}
    enhances the sliding-window approach by replacing the standard Euclidean confidence interval with one based on KL divergence. While retaining the fixed-size window to track non-stationary changes, this modification allows for tighter confidence bounds, particularly in bounded reward settings, leading to more efficient exploration compared to the standard SW-UCB.
    \item 
    SW-TS \citep{trovo2020sliding}
    adapts the standard Thompson Sampling framework to non-stationary environments by incorporating a sliding-window mechanism. Instead of aggregating the entire history, the algorithm updates its Bayesian posterior distributions using exclusively the most recent observations contained within the window. This approach enables the probabilistic policy to track dynamic changes in the underlying reward distributions by effectively discarding outdated information.
    \item Rexp3 \citep{besbes2014stochastic}
    adapts the adversarial Exp3 algorithm to stochastic non-stationary settings by introducing a periodic restart mechanism. The algorithm divides the time horizon into fixed-size epochs and resets the Exp3 weights at the beginning of each epoch. This approach allows the policy to handle distribution changes by effectively limiting its memory to the current epoch, with the restart frequency typically tuned based on the total variation budget of the reward functions.
    \item Ser4 \citep{allesiardo2017non}
    handles non-stationary environments by combining a successive elimination strategy with a periodic reset mechanism. Within each epoch, the algorithm progressively discards arms identified as suboptimal based on their estimated confidence bounds. To adapt to dynamic reward changes, Ser4 periodically resets the set of active arms, thereby reviewing previously eliminated candidates to check if they have become optimal in the new environment.
\end{itemize}



\section{Additional Experimental Results}
\label{sec:additional_experiments}

In this section, we present supplementary experimental results for a fuller picture of the algorithms' behavior. We report the pairwise win rates and average rankings over all baselines, computed across 
100 independent instances with 40 random seeds each, separately for the Linear-Then-Flat (LTF) and Concave settings.

\begin{figure}[p]
    \centering

    \begin{subfigure}{\textwidth}
        \centering
        \includegraphics[width=0.5\linewidth, trim={0.3cm 0.3cm 0.3cm 0.3cm}, clip]{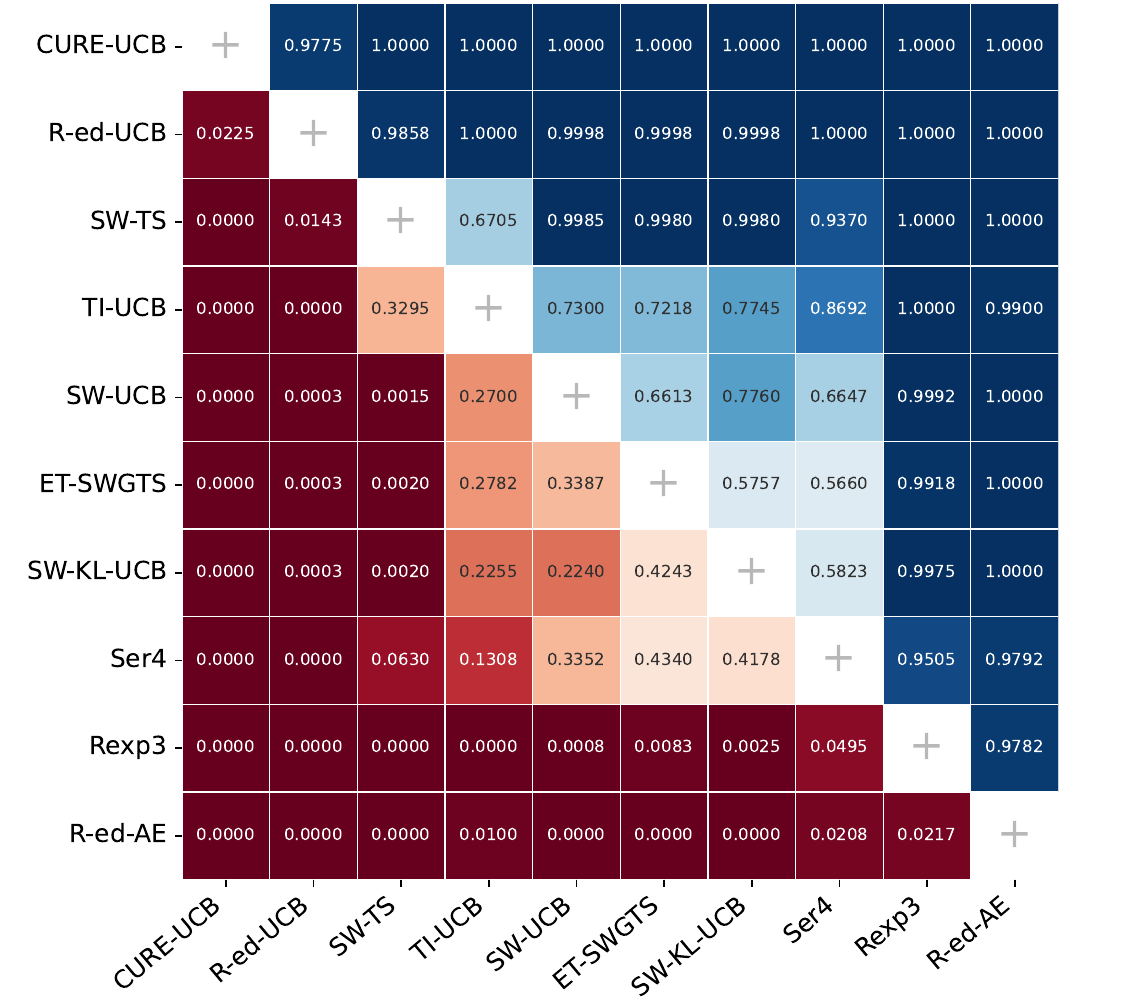}
        \hfill
        \includegraphics[width=0.49\linewidth, trim={0.3cm 0.3cm 0.3cm 0.3cm}, clip]{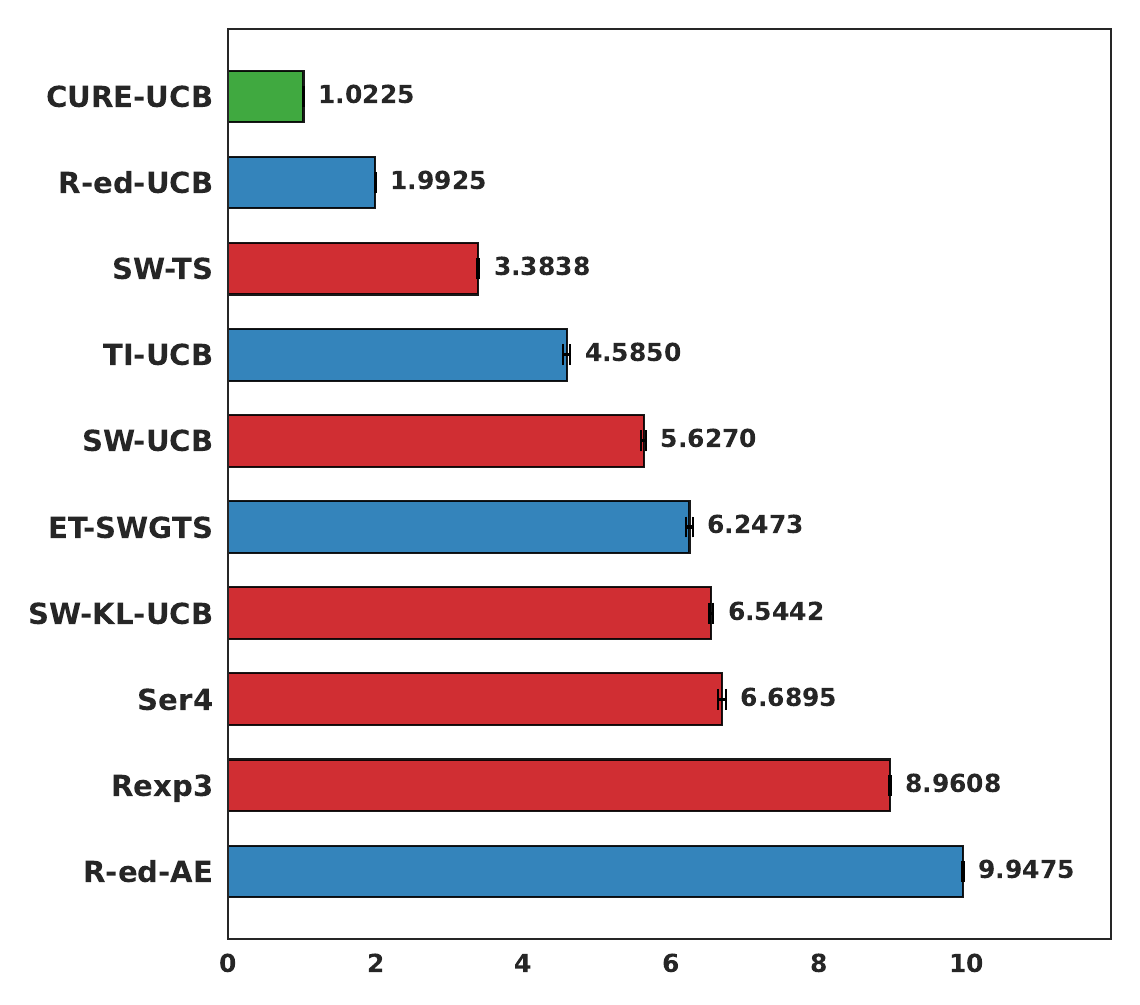}
        \caption{$T = 20K$}
    \end{subfigure}

    \vfill

    \begin{subfigure}{\textwidth}
        \centering
        \includegraphics[width=0.5\linewidth, trim={0.3cm 0.3cm 0.3cm 0.3cm}, clip]{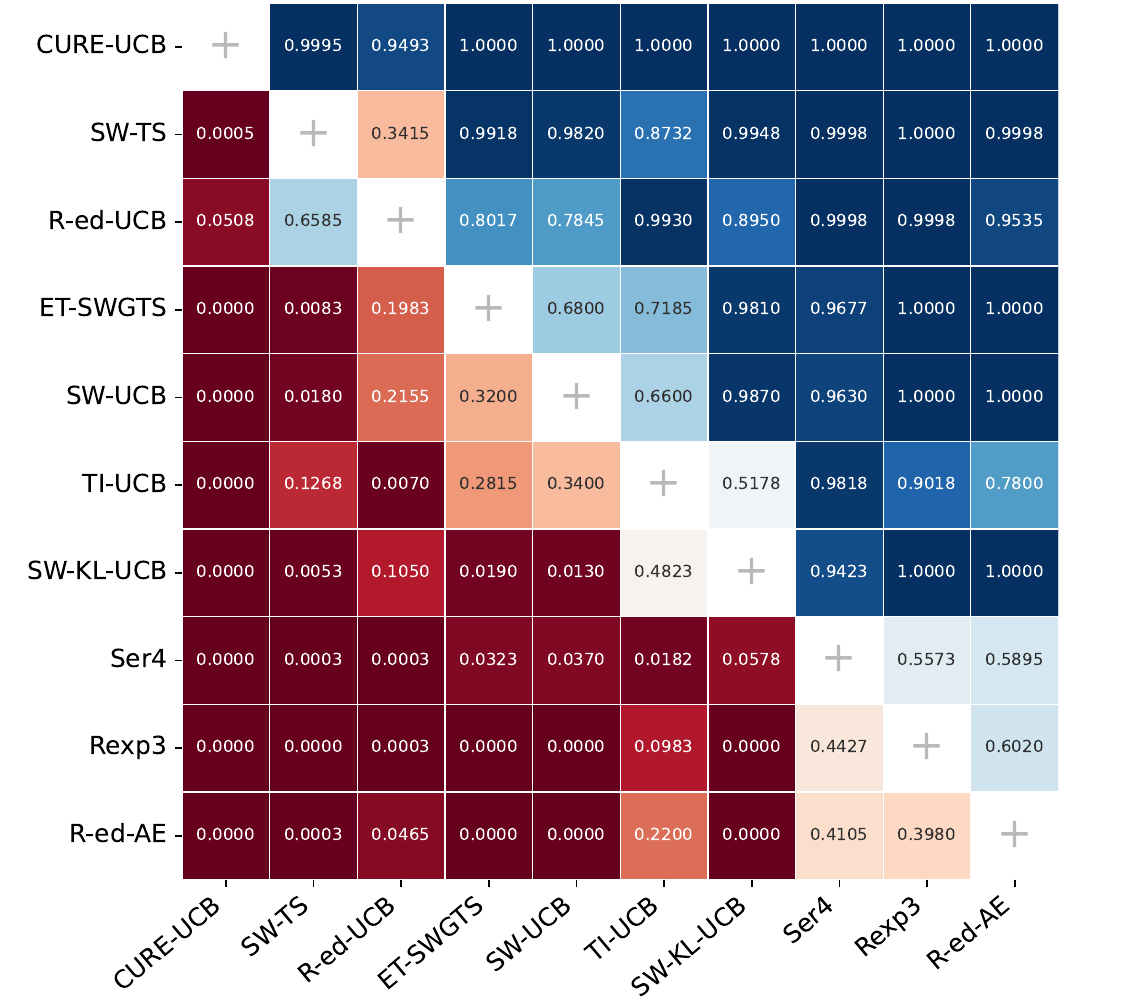}
        \hfill
        \includegraphics[width=0.49\linewidth, trim={0.3cm 0.3cm 0.3cm 0.3cm}, clip]{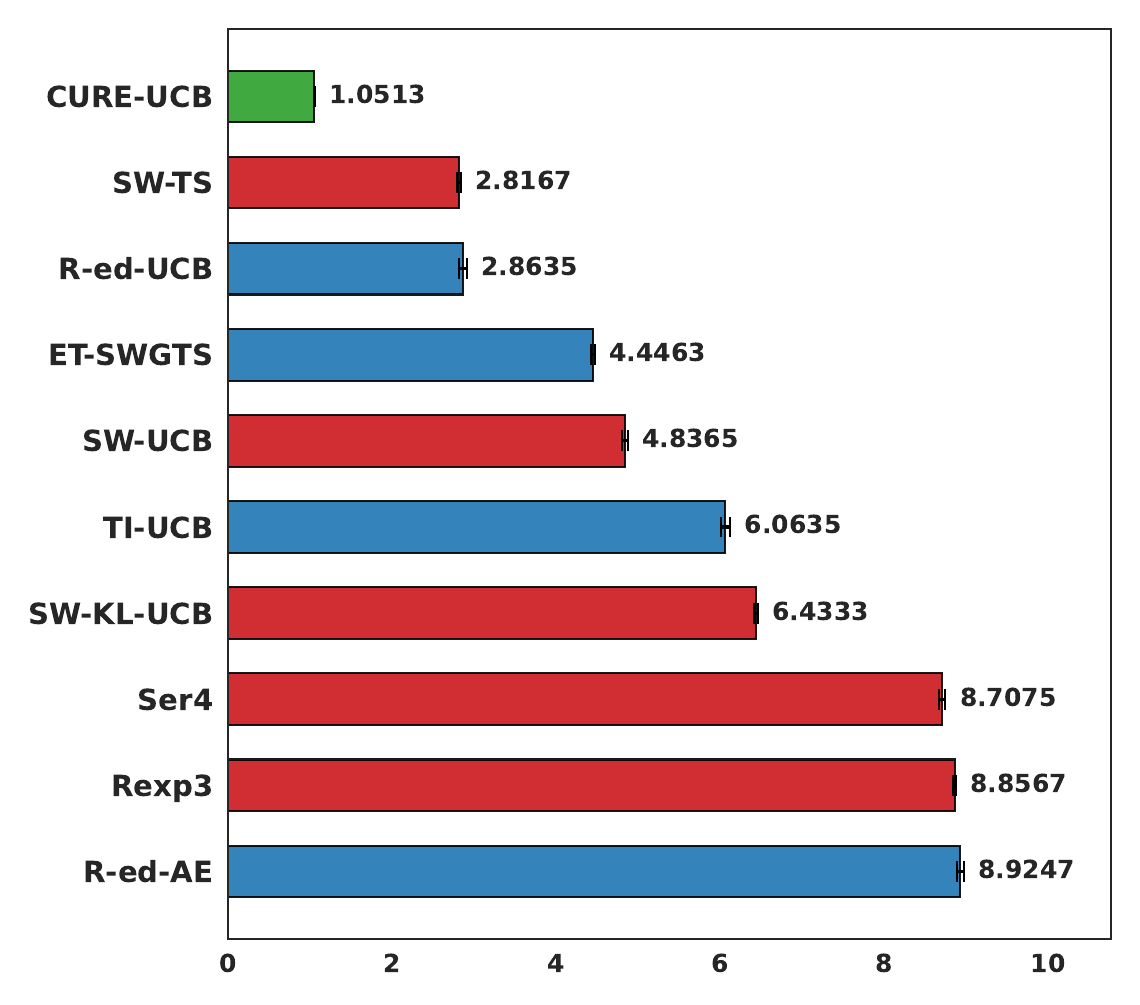}
        \caption{$T = 60K$}
    \end{subfigure}

    \vfill

    \begin{subfigure}{\textwidth}
        \centering
        \includegraphics[width=0.5\linewidth, trim={0.3cm 0.3cm 0.3cm 0.3cm}, clip]{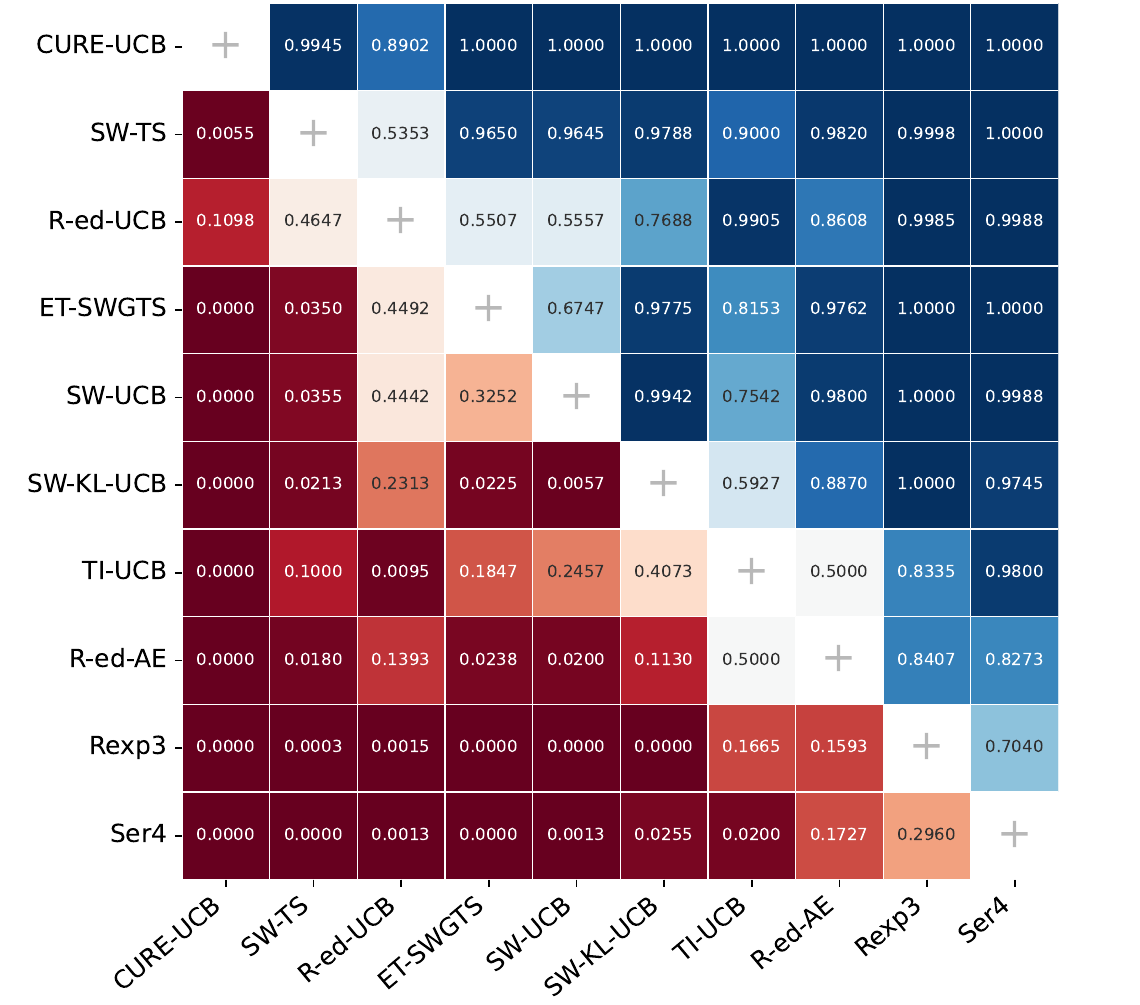}
        \hfill
        \includegraphics[width=0.49\linewidth, trim={0.3cm 0.3cm 0.3cm 0.3cm}, clip]{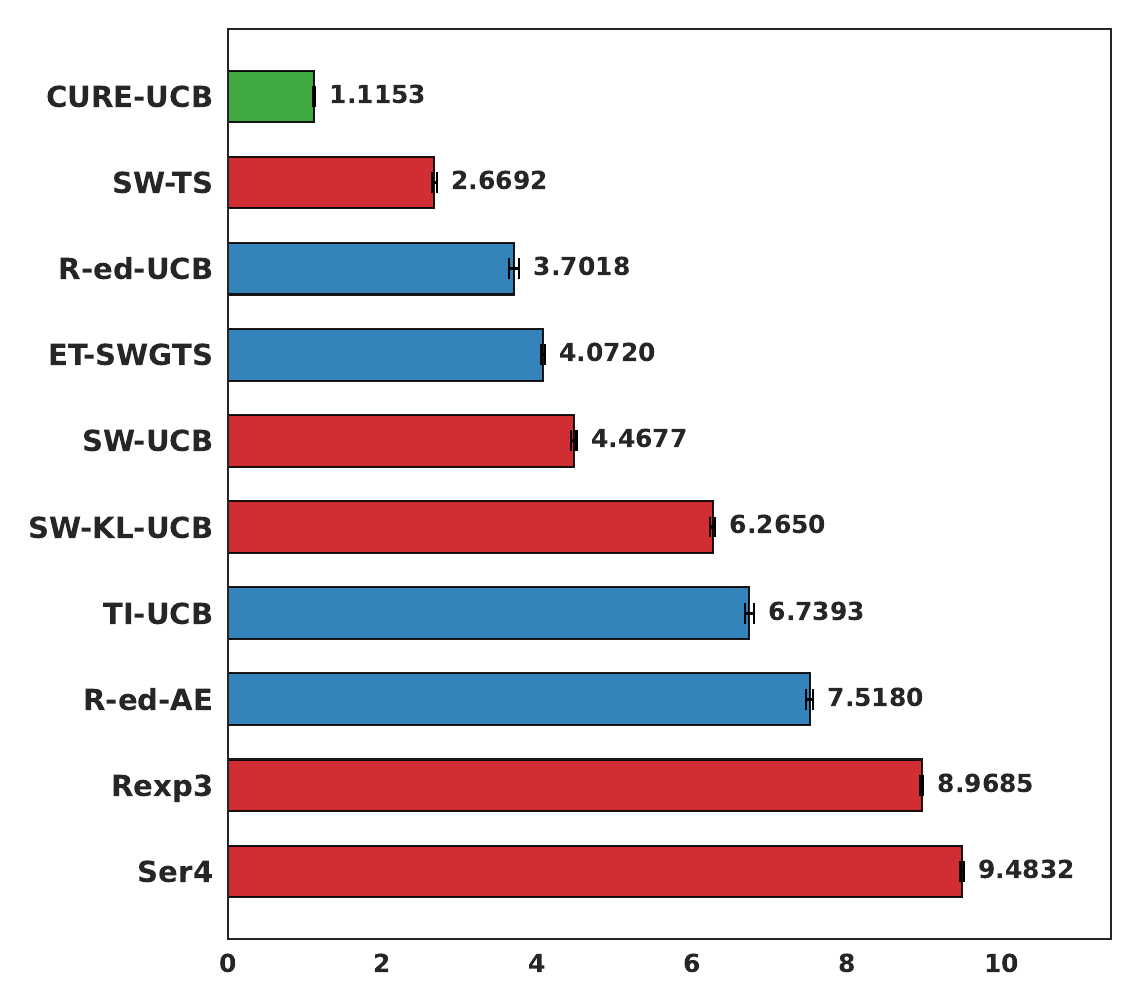}
        \caption{$T = 100K$}
    \end{subfigure}

    \caption{\textbf{Full results for the LTF setting.} (Left) Pairwise win-rate heatmap over all baselines. (Right) Average ranking over all baselines. Both are computed over $100$ independent instances at $T=20K$, $60K$, and $100K$.}
    \label{fig:main_results_detailed1}
\end{figure}

\begin{figure}[p]
    \centering

    \begin{subfigure}{\textwidth}
        \centering
        \includegraphics[width=0.5\linewidth, trim={0.3cm 0.3cm 0.3cm 0.3cm}, clip]{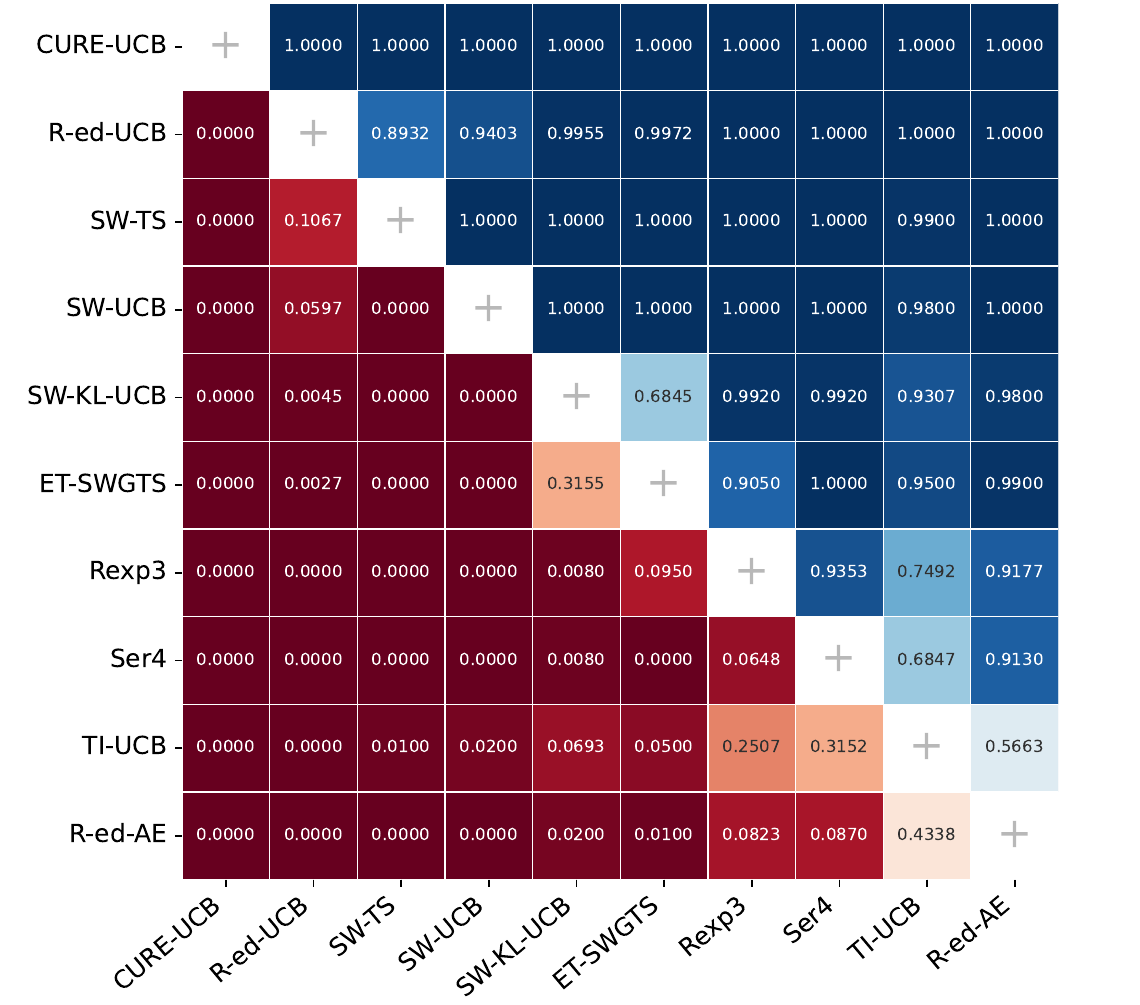}
        \hfill
        \includegraphics[width=0.49\linewidth, trim={0.3cm 0.3cm 0.3cm 0.3cm}, clip]{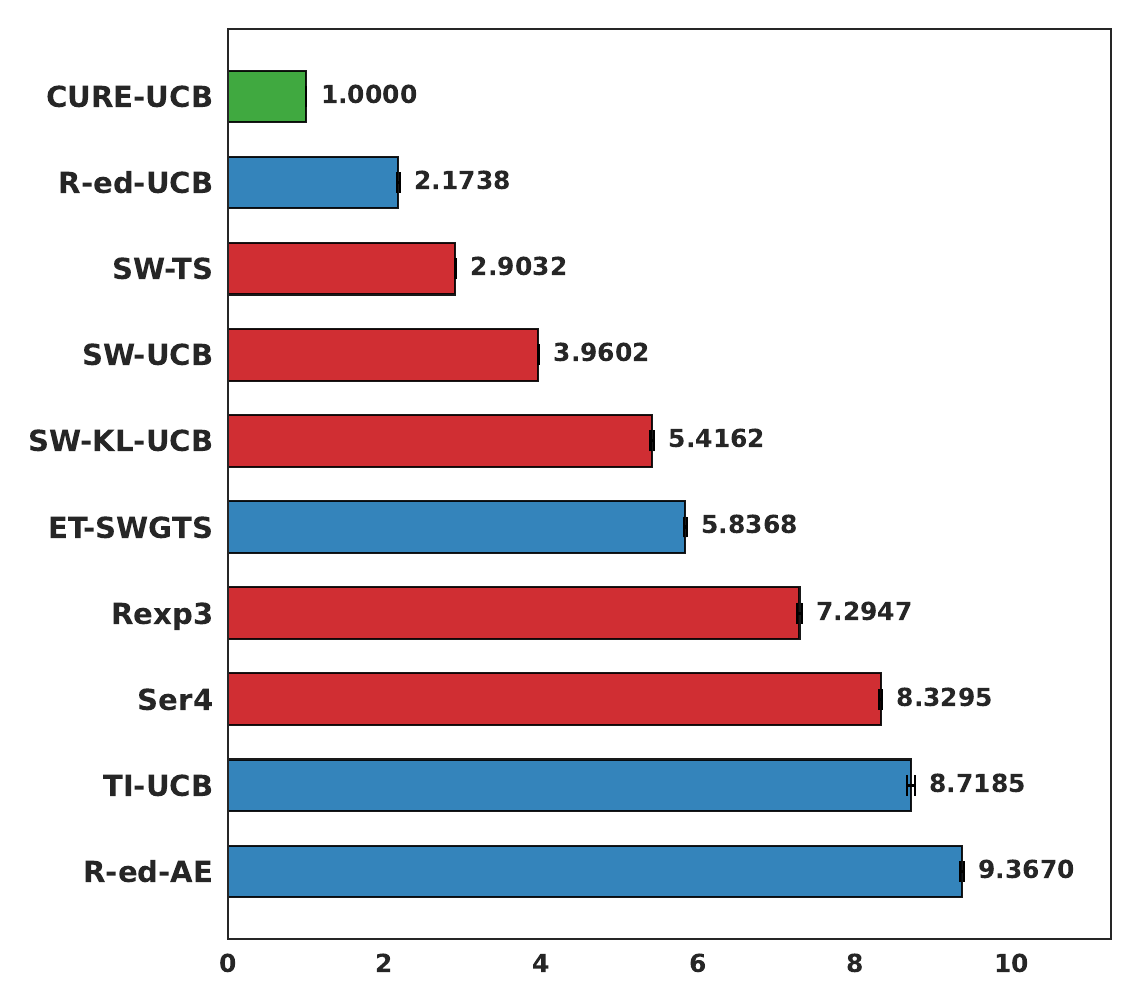}
        \caption{$T = 20K$}
    \end{subfigure}

    \vfill

    \begin{subfigure}{\textwidth}
        \centering
        \includegraphics[width=0.5\linewidth, trim={0.3cm 0.3cm 0.3cm 0.3cm}, clip]{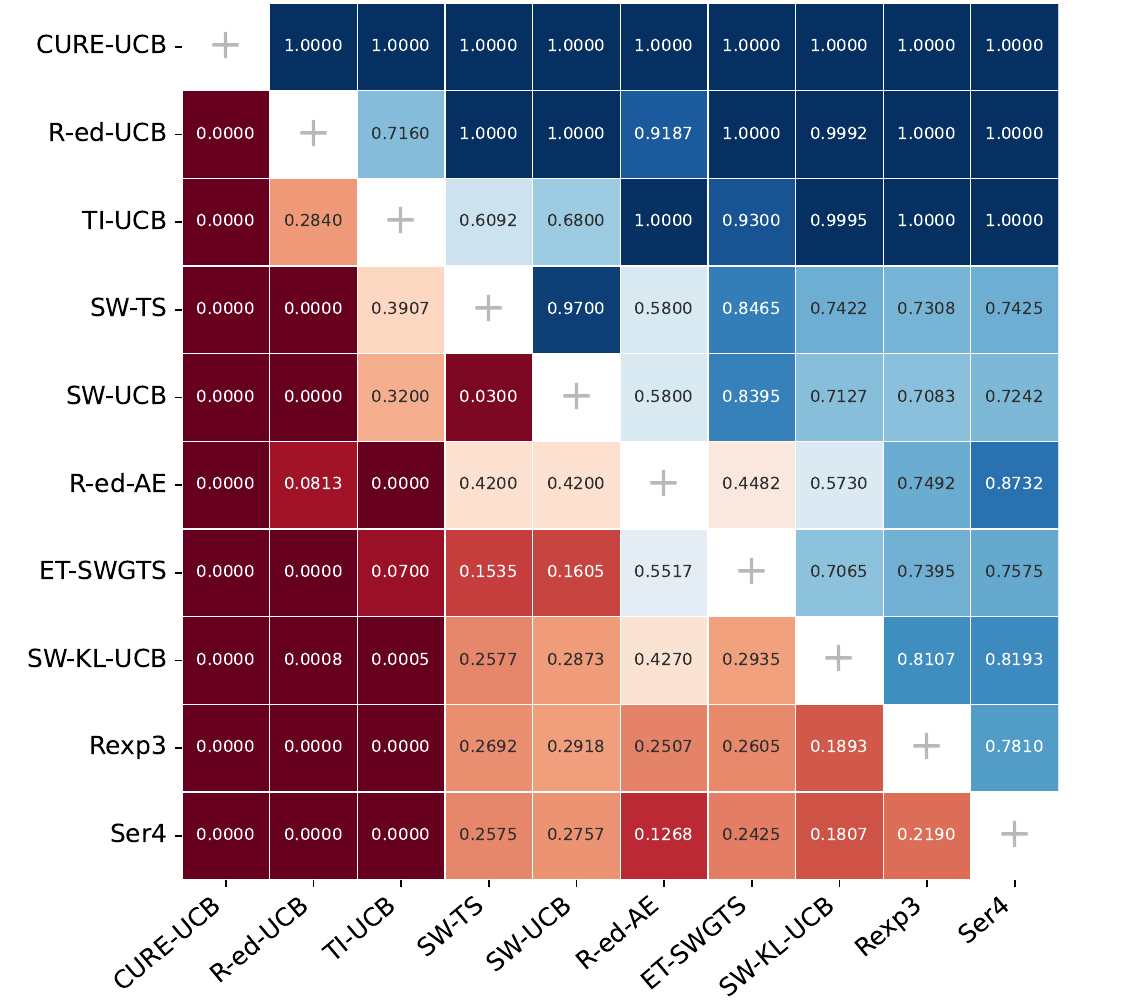}
        \hfill
        \includegraphics[width=0.49\linewidth, trim={0.3cm 0.3cm 0.3cm 0.3cm}, clip]{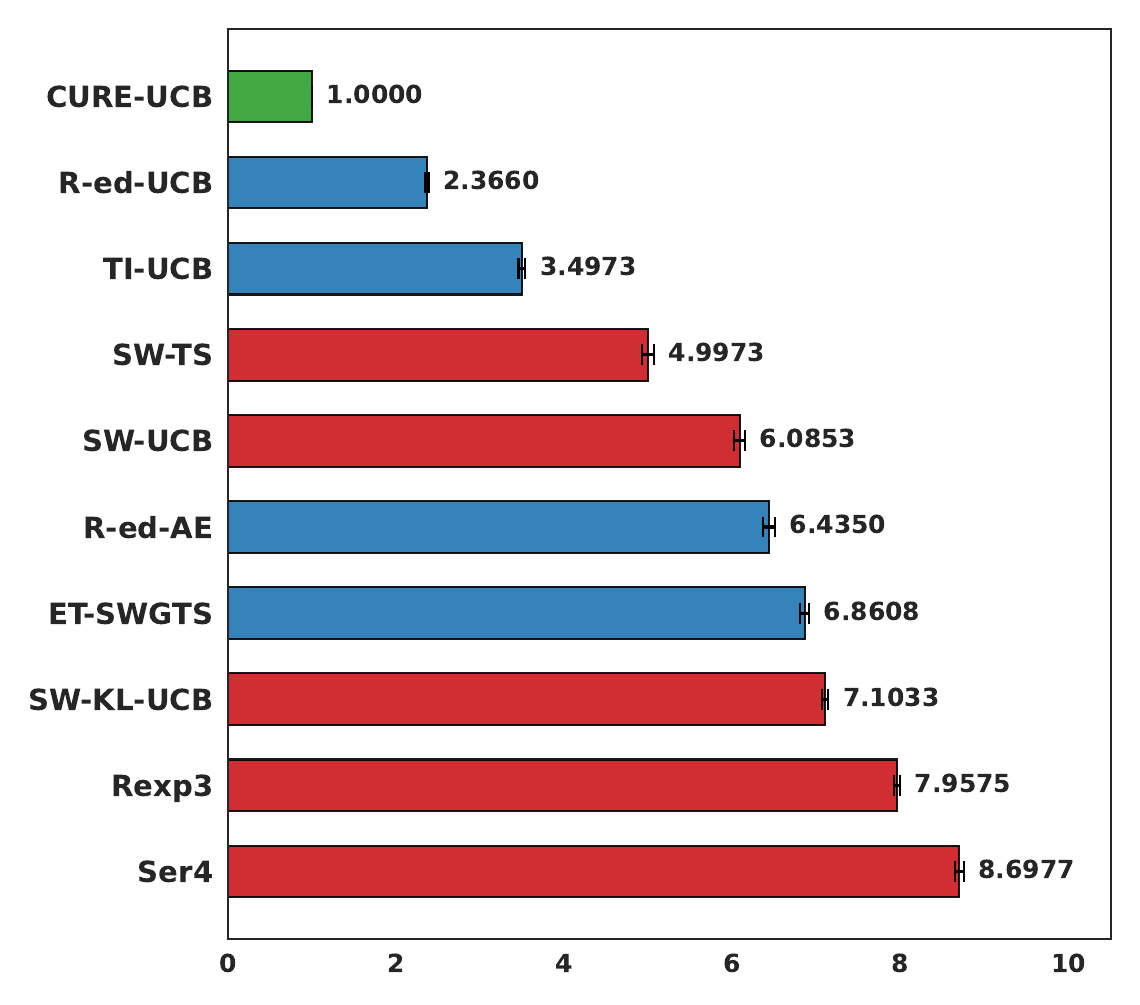}
        \caption{$T = 60K$}
    \end{subfigure}

    \vfill

    \begin{subfigure}{\textwidth}
        \centering
        \includegraphics[width=0.5\linewidth, trim={0.3cm 0.3cm 0.3cm 0.3cm}, clip]{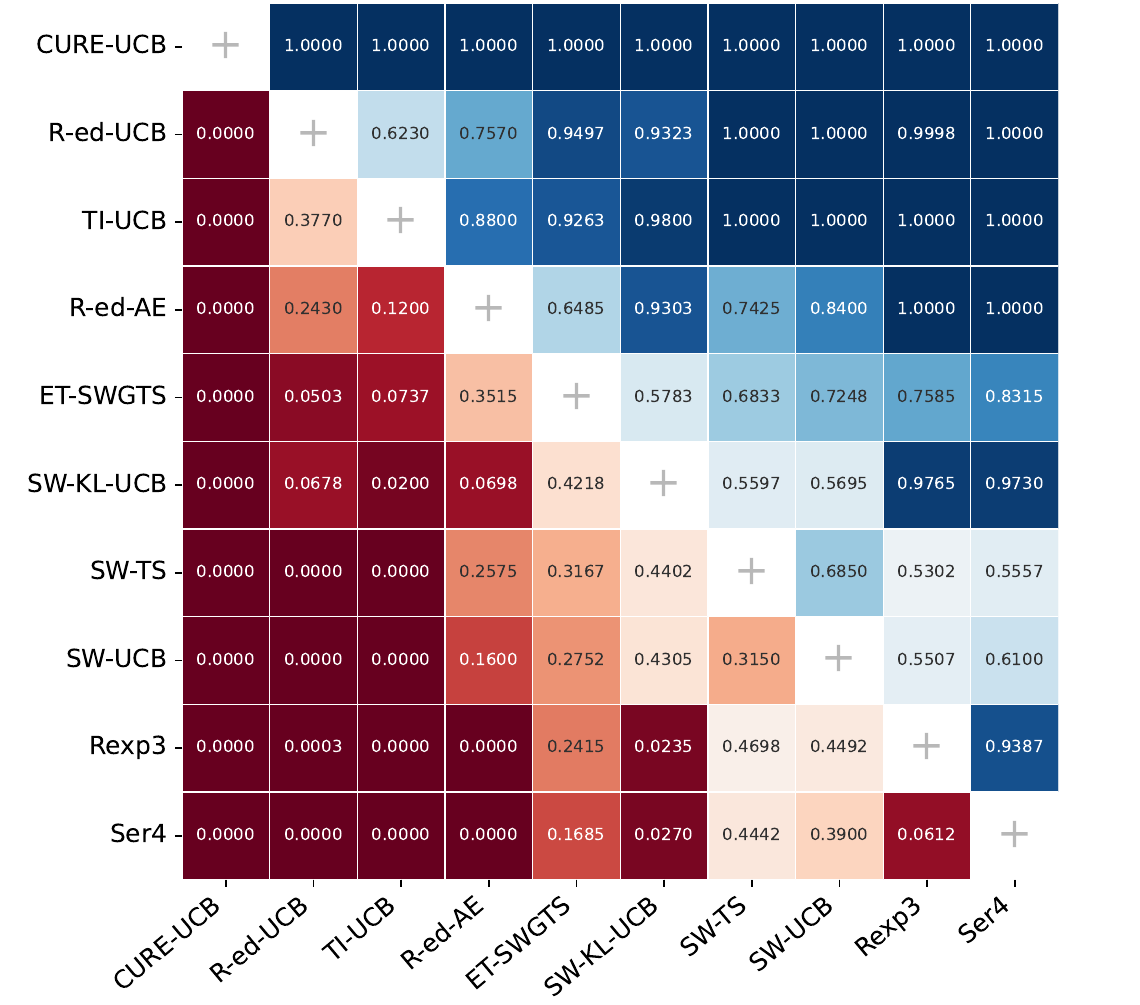}
        \hfill
        \includegraphics[width=0.49\linewidth, trim={0.3cm 0.3cm 0.3cm 0.3cm}, clip]{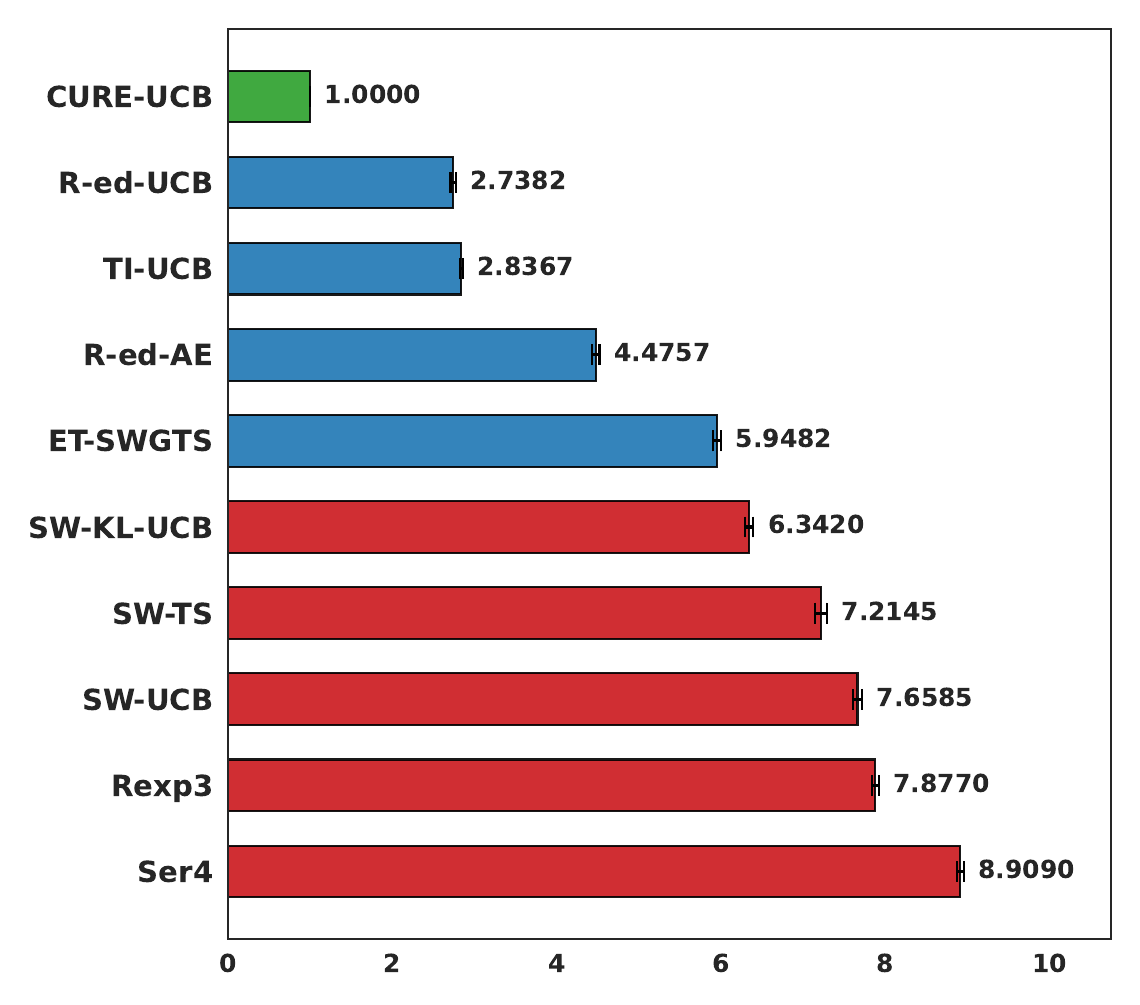}
        \caption{$T = 100K$}
    \end{subfigure}

    \caption{\textbf{Full results for the Concave setting.} (Left) Pairwise win-rate heatmap over all baselines. (Right) Average ranking over all baselines. Both are computed over $100$ independent instances at $T=20K$, $60K$, and $100K$.}
    \label{fig:main_results_detailed_concave}
\end{figure}
\clearpage

\section{Additional experiments for robustness}
\label{sec:robustness}
First, we examine sensitivity to the window-size parameter $\varepsilon$, which governs the bias-variance tradeoff in growth-rate estimation (Appendix~\ref{app:window-size}). 
Second, we test robustness to varying levels of reward noise (Appendix~\ref{app:noise-robustness}). Across all studies, CURE-UCB consistently maintains its advantage over horizon-agnostic baselines.

\subsection{Sensitivity to Window Size  \texorpdfstring{$h_i$}{h\_i}}
\label{app:window-size}

The CURE-UCB index uses a sliding window of size $h_i$ to estimate the recent reward and growth rate. We follow the standard choice $h_i:=\max\left\{\lfloor \epsilon N_{i,t-1} \rfloor,1\right\}$ with $\varepsilon \in [0, 0.5]$, controlling the bias-variance tradeoff: larger windows reduce variance but introduce more bias from older samples. To assess sensitivity to this hyperparameter, we vary $\varepsilon \in \{1/2, 1/4, 1/8, 1/16, 1/32\}$ at $T=50{,}000$.

Table~\ref{tab:window-size} reports the mean cumulative regret in both LTF and Concave settings. Performance degrades gradually as the window shrinks, with no abrupt failure at any particular value. Larger windows yield better performance because they provide more samples for cumulative reward estimation, reducing variance.

\begin{table}[h]
\centering
\caption{Mean cumulative regret of CURE-UCB under varying window-size parameter $\varepsilon$ at $T=50{,}000$, where $h_i = \max\left\{\lfloor \varepsilon N_{i,t-1} \rfloor,1\right\}$. }
\label{tab:window-size}
\begin{tabular}{lcc}
\toprule
$\varepsilon$ & Concave & LTF \\
\midrule
1/2  & 2{,}056 & 3{,}230 \\
1/4  & 2{,}818 & 4{,}576 \\
1/8  & 3{,}691 & 6{,}574 \\
1/16 & 4{,}766 & 9{,}158 \\
1/32 & 6{,}069 & 12{,}254 \\
\bottomrule
\end{tabular}
\end{table}

\subsection{Robustness to noisy rewards}
\label{app:noise-robustness}

Rising bandit algorithms estimate growth rates from differences between past samples, making the choice of noise level relevant. Since these estimates are aggregated over multiple observations within the sliding window, the effect of individual noisy samples is averaged out, allowing the algorithm to capture the underlying growth pattern despite per-sample noise. To empirically validate this, we evaluate CURE-UCB across a range of noise levels at $T=50{,}000$.

Table~\ref{tab:noise-robustness} reports the mean cumulative regret in both LTF and Concave settings, with noise standard deviation $\sigma \in \{0.001, 0.01, 0.05, 0.1\}$. While CURE-UCB and R-ed-UCB both incorporate $\sigma$ in their index design, causing their regret to vary with the noise level, CURE-UCB consistently maintains a significant advantage over all baselines across all tested noise levels, confirming its robustness under noisy reward dynamics.

\begin{table}[h]
\centering
\caption{Mean cumulative regret under varying noise levels $\sigma$ at $T=50{,}000$.}
\label{tab:noise-robustness}
\begin{tabular}{lcccc cccc}
\toprule
& \multicolumn{4}{c}{Concave} & \multicolumn{4}{c}{LTF} \\
\cmidrule(lr){2-5}\cmidrule(lr){6-9}
Algorithm & $\sigma{=}0.001$ & $\sigma{=}0.01$ & $\sigma{=}0.05$ & $\sigma{=}0.1$ 
          & $\sigma{=}0.001$ & $\sigma{=}0.01$ & $\sigma{=}0.05$ & $\sigma{=}0.1$ \\
\midrule
CURE-UCB & \textbf{783}   & \textbf{2{,}057} & \textbf{3{,}439} & \textbf{4{,}145}
         & \textbf{1{,}974} & \textbf{3{,}239} & \textbf{5{,}450} & \textbf{6{,}981} \\
SW-UCB   & 10{,}212 & 10{,}209 & 10{,}209 & 10{,}176
         & 7{,}855  & 7{,}837  & 7{,}896  & 7{,}890 \\
R-ed-UCB & 6{,}667  & 7{,}886  & 8{,}059  & 8{,}530
         & 8{,}761  & 10{,}268 & 13{,}805 & 15{,}585 \\
\bottomrule
\end{tabular}
\end{table}

\end{document}